\documentclass[journal, 10pt, twocolumn]{IEEEtran}
\usepackage{amsmath,amssymb,pstricks,color,dsfont,amsthm}
\usepackage{graphicx}
\usepackage{psfrag,graphicx,epsfig}
\usepackage{epstopdf}
\usepackage{xspace}
\usepackage{subfig}
\usepackage{float}
\usepackage{placeins}
\usepackage{multirow}
\usepackage{pgf,tikz}
\usepackage{nowidow}
\usepackage{flushend}
\usepackage{enumitem}
\usepackage{mathrsfs}
\usepackage[vlined,linesnumbered,boxruled]{algorithm2e}
\usepackage[switch,columnwise]{lineno}
\IEEEoverridecommandlockouts
\usepackage{algpseudocode}

\usepackage{etex}
\usepackage{array}
\usepackage{cite}
\newcolumntype{x}[1]{%
	>{\centering\hspace{0pt}}p{#1}}%
\usepackage{url}

\newcommand{\Epi}{\mathcal{E}pi}
\DeclareMathOperator*{\Lim}{Lim}
\DeclareMathOperator*{\Limsup}{Limsup}

\newcommand{\EQ}{\begin{eqnarray}}
	\newcommand{\EN}{\end{eqnarray}}
\newcommand{\EQQ}{\begin{eqnarray*}}
	\newcommand{\ENN}{\end{eqnarray*}}

\newcommand{\real}{\mathbb{R}}

\renewcommand{\natural}{\mathbb{N}}
\newcommand{\natzero}{\mathbb{N}_{\geq0}}

\newcommand{\BB}{\mathcal{B}}
\newcommand{\CC}{\mathcal{C}}
\newcommand{\DD}{\mathcal{D}}

\newcommand{\FF}{\mathcal{F}}

\newcommand{\HH}{\mathcal{H}}
\newcommand{\II}{\mathcal{I}}

\newcommand{\TT}{\mathcal{T}}
\newcommand{\UU}{\mathcal{U}}
\newcommand{\XX}{\mathcal{X}}
\newcommand{\YY}{\mathcal{Y}}
\newcommand{\ZZ}{\mathcal{Z}}

\newcommand{\aA}{\mathfrak{A}}
\newcommand{\bB}{\mathfrak{B}}
\newcommand{\cC}{\mathfrak{C}}
\newcommand{\dD}{\mathfrak{D}}

\newcommand{\tT}{\mathfrak{T}}

\newcommand{\dsT}{\mathds{T}}

\newcommand{\dist}{\operatorname{dist}}

\newcommand{\argmax}{\operatorname{argmax}}

\newtheorem{theorem}{\bf Theorem}[section]
\newtheorem{corollary}{\bf Corollary}[section]
\newtheorem{lemma}{\bf Lemma}[section]

\newtheorem{remark}{\bf Remark}[section]

\newtheorem{assumption}{\bf Assumption}[section]
\newtheorem{claim}{\bf Claim}[section]

\title{
iPolicy: Incremental Policy Algorithms for Feedback Motion Planning
}
\begin{document}

\author{\begin{tabular}{cccc}
    Guoxiang Zhao$^{1, *}$ & Devesh K. Jha$^{2, *}$ & Yebin Wang$^2$  & Minghui Zhu$^3$\\	 
	\end{tabular}\\ \vspace{3pt}
    \thanks{$^*$ Equal contributions.}
    \thanks{$^1$ Guoxiang Zhao is with School of Future Technology and Institute of Artificial Intelligence, Shanghai University, Shanghai, 200444, China. (email: \texttt{gxzhao@shu.edu.cn})}
	\thanks{$^2$ Devesh K. Jha and Yebin Wang are with Mitsubishi Electric Research Laboratories, Cambridge, MA 02139, USA. (email: \texttt{jha@merl.com}, \texttt{yebinwang@ieee.org})}
	\thanks{$^3$ Minghui Zhu is with School of Electrical Engineering and Computer Science, Pennsylvania State University, University Park, PA 16802, USA. This work by M. Zhu was partially supported by the grant U.S. NSF CNS1830390. (email: \texttt{muz16@psu.edu})}
}

\maketitle
\thispagestyle{empty}
\pagestyle{empty}

\begin{abstract}

This paper presents policy-based motion planning for robotic systems. The motion planning literature has been mostly focused on open-loop trajectory planning which is followed by tracking online. In contrast, we solve the problem of path planning and controller synthesis simultaneously by solving the related feedback control problem. We present a novel incremental policy (iPolicy) algorithm for motion planning, which integrates sampling-based methods and set-valued optimal control methods to compute feedback controllers for the robotic system. In particular, we use sampling to incrementally construct the state space of the system. Asynchronous value iterations are performed on the sampled state space to synthesize the incremental policy feedback controller. We show the convergence of the estimates to the optimal value function in continuous state space. Numerical results with various different dynamical systems (including nonholonomic systems) verify the optimality and effectiveness of iPolicy. 

\end{abstract}

\section{Introduction}\label{sec:Intro}
\IEEEPARstart{I}{nformally} speaking, given a robot with a description of its dynamics, a description of its environment and a set of goal states, the motion planning problem is to find a sequence of control inputs so as to guide the robot from the initial state to one of the goal states while avoiding collision in the cluttered environment. It is well-known that robotic motion planning is at least as difficult as the generalized piano mover's problem, which has been proven to be PSPACE-hard~\cite{JHR:79}. 
Many planning algorithms have been proposed.
The discrete approaches, such as Dijkstra's Algorithm and A*~\cite{L06}, usually depend on the structured environment and apply graph search to find the shortest path.
The optimality of their returned trajectories are formally guaranteed for discrete problems but their heavy dependency on structured environment makes them suffer from inaccuracies and infeasibility of the solution when dealing with continuous systems~\cite{L11}.
Recently, sampling-based geometric planning algorithms such as the rapidly-exploring random trees (RRT)~\cite{LK01} and its optimal variant RRT*~\cite{KF11} are arguably the most influential and widely-used motion planning algorithms since the last two decades. They are shown to compute quickly in high-dimensional continuous space and possess theoretical guarantees such as probabilistic completeness and optimality. 

Robots, in most of the practical applications, have stringent differential constraints on their motion which need to be properly incorporated during motion planning. Motion planning for dynamical systems has a rich history and much work has been done on this topic~\cite{MS93,LSL98,L06, wang2017two}.  However, the problem of motion planning for nonholonomic systems is still open in many aspects~\cite{schmerling2015optimal}. Sampling-based algorithms have received a lot of attention for their efficiency in handling obstacle-cluttered environments.
There are two main bodies of research to make these methods more suitable for dealing with differential constraints: one develops steering functions for dynamical systems using concepts from feedback control theory (see e.g.,~\cite{PPKKL12, WV13}); the other direction of work improves the computational effectiveness of the algorithms using geometric planning (see e.g.,~\cite{schmerling2015optimal, KF13, hauser2015lazy}). Despite the tremendous body of work on this topic, most of the work focuses on planning open-loop trajectories.

In practical implementation, the planned open-loop trajectories are tracked by a feedback controller.~\cite{8718025}.
The presence of differential, state and input constraints, however, makes the design of the feedback tracking controller for highly nonlinear and complex robots still a difficult problem.
Furthermore, the performance guarantees of the closed-loop trajectories could be lost in the presence of these constraints, especially in terms of minimizing the aggregated cost. 
It necessitates feedback motion planning which explicitly takes into account the feedback tracking during the planning process. 
The feedback planner computes the mapping from state space to control space subject to constraints and equivalently searches for the solution for every initial condition of the robot; this facilitates the completeness and performance guarantee of the planner.
However, the feedback motion planning problem is challenging in various aspects and most of the computational issues involving feedback planning are still unexplored and open. Interested readers are referred to~\cite{L11} for a comprehensive survey of the challenges in feedback motion planning. Most of the previous approaches for feedback motion planning struggle with computational tractability in continuous state-space or local minima issues~\cite{L06}.
In this paper, we present a method for feedback motion planning of dynamical systems which makes use of sampling-based methods and value iterations to recover the optimal feedback controller. 
As a synthesis of the computational advantages of sampling-based algorithms and the approximation consistency of set-valued algorithms \cite{cardaliaguet1999set}, the proposed iPolicy algorithm is an incremental and anytime feedback motion planner with formal guarantee of asymptotic optimality.

\textbf{Contributions.} 
This paper presents a sampling-based algorithm for feedback motion planning for a class of nonlinear robotic systems. 
In this work, we leverage the existing set-valued analysis tool for motion planning and asynchronous value iterations to synthesize feedback motion planners in an incremental fashion. 
In particular, the continual sampling creates an approximation of the original minimal time problem in every iteration and the policy on the discretized space provides an incrementally refined estimate of the value function. 
Using value iterations in an asynchronous and incremental manner limits the computational requirements while retaining optimality conditions. 
We show the convergence of the estimated value functions to the optimal value function for the motion planning problem in the continuous state space. 
For clarification of presentation, the various salient features of the algorithms are demonstrated with simulations using point mass system. 
Some numerical simulations are then provided for feedback motion planning of simple car and Dubins car (these are well studied nonholonomic systems~\cite{L06, wang2017two}) using the proposed motion planning algorithm. 
Several simulations are provided to show the value functions calculated by the proposed algorithm and trajectories under the guidance of the computed controller in cluttered environments. 
Through simulation study of various different dynamical systems, we show the applicability of the proposed algorithms across a wide range of robotic systems (two different classic nonholonomic systems). 

\textbf{Organization.} This paper is organized in eight sections including the current one. We present work related to our proposed problem in Section~\ref{sec:related_work}. Commonly sued notations and notions are clarified in Section~\ref{sec:Notations}. In Section~\ref{sec:ProblemStatement}, we present a formal statement of the problem solved in this paper. The main algorithms of the paper are presented in Section~\ref{sec:Algorithm} followed by analysis of the same in Section~\ref{sec:Analysis}. Numerical results for the proposed algorithm are presented in Section~\ref{sec:Results} with related interpretation and discussion. The paper is finally concluded with a summary and future work in Section~\ref{sec:conclusions}.

\section{Related Work}\label{sec:related_work}
Motion planning is central to robotics. Consequently, it has received a lot of attention in robotics and controls community. Broadly speaking, there are mostly three kinds of approaches for motion planning -- optimization-based methods, gradient-based methods and sampling-based methods. 

Some of the most popular optimization-based approaches could be found in~\cite{5980280, 5152817, doi:10.1137/16M1062569, schulman2014motion, DBLP:journals/corr/abs-2106-03220}. The main idea of these approaches is to formulate a dynamic optimization problem with smooth formulation of constraints (like collision, etc.). The resulting optimization problem can then be solved to generate an optimal trajectory. However, the main shortcoming of these optimization-based techniques is that they struggle to find solution in obstacle cluttered, non-convex spaces. Trajectory optimization techniques using optimal control literature can generate control trajectories in presence of state and input constraints for nonlinear dynamical systems -- however, they may not be able to find feasible trajectories in the presence of arbitrary obstacles ~\cite{doi:10.1002/rnc.4488, tassa2012synthesis}. These methods are mostly used to compute trajectories under dynamical constraints which can be followed online using a trajectory tracking controller using desirable constraints for execution~\cite{zhang2021trajectory, de2020tunempc, majumdar2017funnel}. 

Gradient-based approaches, such as potential fields \cite{khatib1986realtime, BOREN91} and navigation functions \cite{koditschek1990robot}, consider the composition of a repulsive component for collision avoidance and an attractive component for reaching goal, and derive the control law based on the gradient of the synthesized field.
The computations of both fields usually only depend on local information and this brings computational efficiency to gradient-based approaches; however, the cancellation of different components in the composition can also trap the robot at local minima~\cite{L06}. Thus, these approaches tend to struggle when a robot has to operate in obstacle cluttered environment.

Sampling-based approaches are shown to be able to successfully address irregularly shaped environments and arguably the most widely used methods for motion planning in robotics~\cite{L06, doi:10.1146/annurev-control-061623-094742}. However, most of these methods are used to compute open-loop plans for the robot to follow.
It is, in general, desirable that robots operate in a feedback fashion using state estimates during execution of a task. Consequently, there has been some work on using sampling-based algorithms to develop feedback planners by using different metrics to select the best path. Some methods to solve for feedback planners are presented in~\cite{TMTR10, T09, PPKKL12, MSD11, HKF12}. 
However, very little work has been done on using these algorithms to incrementally build the state-space of systems to solve for continuous time and space optimal control problems for the dynamical systems. So, even though the motion planning problem has been thoroughly studied in literature, there seems to be lack of techniques and algorithms which can compute feedback motion planners for dynamical systems. We make an attempt to address this problem in this paper.

Our approach is closely related to the value iterations-based approach in~\cite{JZR15A}; however, we use asynchronous value iterations, provide results for nonholonomic systems and also provide rigorous numerical simulation results.
Our results on optimal performance do not require the explicit description of steering functions between any two states of the system in the collision-free space. 
As such, we expect our results to cater to a rich set of robotic motion planning problems. 
Our approach is also related to \cite{zhao2020pareto} but differs in the way to solve the approximate motion planning problem, where \cite{zhao2020pareto} spends heavy computations searching for the optimal solution over the entire graph before refinement while our approach swiftly solves for a small subset of the graph and proceeds to finer graphs.
This brings the continuously increasing optimality to our approach and it incrementally improves solutions once more computational resources are given.
The problem and algorithms presented in this paper are different from some other sampling-based feedback planning like~\cite{ACM14}, and the proposed algorithms have the advantage of being incremental (instead of batch). Furthermore, optimal performance guarantees for the feedback control problem have been provided.

\section{Notations and notions}\label{sec:Notations}
Let $\|\cdot\|$ be the $2$-norm in $\real^n$.
Define the distance from a point $x\in\real^n$ to a set $\XX\subseteq\real^n$ by $\dist(x, \XX)\triangleq\inf_{x'\in\XX}\|x-x'\|$.
Given a compact set $\XX\subseteq\real^n$ and a function $v:\XX\to\real$, denote the supremum norm over $\XX$ by $\|v\|_\XX\triangleq\sup_{x\in\XX}|v(x)|$.
Let the Lebesgue measure of the set $\XX$ be $\mu(\XX)$.
Denote the unit ball in $\real^n$ by $\BB_n$ and the volume of the unit ball by $C_n\triangleq\frac{\pi^{n/2}}{(n/2)!}$.
When no ambiguity is caused, the dimensionality in the subscript may be omitted.
Define the Minkowski sum in $\real^n$ for $\XX\subseteq\real^n$ and $\YY\subseteq\real^n$ by $\XX+\YY\triangleq\{x+y|x\in\XX, y\in\YY\}$.
When $\XX=\{x\}$ is a singleton, with slight abuse of notations, the Minkowski sum can be written as $x+\YY\triangleq\{x+y|y\in\YY\}$.
Define the value assignment operator by $a\leftarrow b$ where the value of $b$ is assigned to $a$.

\section{Problem Formulation}\label{sec:ProblemStatement}

Consider a robot associated with a time-invariant dynamic system governed by the following differential equation:
\begin{equation}\label{eqn:systemeqn}
\dot{x}(t) = f(x(t), u(t)),
\end{equation}
where $x(t)\in
\XX\subseteq\real^{n}$ is the state and $u(t)\in\UU\subseteq\real^m$ is the control of robot. 
The following mild assumptions are imposed throughout the paper and they are inherited from \cite{cardaliaguet1999set}.

\begin{assumption}\label{asmp:system}
    The following assumptions hold for \eqref{eqn:systemeqn}:
	\begin{enumerate}[label=\textbf{(A\arabic*)}, leftmargin = *, align=left]
	\item \label{asmp:compactness}$\XX$ and $\UU$ are non-empty and compact;
	\item \label{asmp:continuity}$f$ is continuous in $u$ and Lipschitz continuous in $x$ for any $u\in\UU$ with a Lipschitz constant $l$;
	\item $f$ is linear growth; i.e., $\exists c\geq0$ s.t. $\forall x\in\XX, u\in\UU, \|f(x, u)\|\leq c(\|x\|+\|u\|+1)$;
	\item \label{asmp:convexity}$\bigcup_{u\in\UU}f(x, u)$ is a convex set for any $x\in\XX$.
	\end{enumerate}
\end{assumption}
Assumptions~\ref{asmp:compactness} and \ref{asmp:continuity} imply $\|f(x, u)\|$ is bounded.
Denote the upper bound of $\|f(x, u)\|$ by $M\triangleq\max_{x\in\XX, u\in\UU}\|f(x, u)\|$.

Let $\XX_{\mathrm{obs}}$ and $\XX_{\mathrm{goal}}$ be the obstacle region and the goal region, respectively. 
Define the obstacle free region as $\XX_{\mathrm{free}}\triangleq\XX\setminus\XX_{\mathrm{obs}}$. 
Denote the trajectory of system~\eqref{eqn:systemeqn} as $\phi(\cdot;x,\pi):\real\to\XX$ given the initial state $x$ and a state feedback controller $\pi:\XX\to\UU$.
Let $\dsT^\pi(x)$ be the first time when the trajectory $\phi(\cdot;x,\pi)$ hits $\XX_{\mathrm{goal}}$ while staying in $\XX_{\mathrm{free}}$ before $\dsT^\pi(x)$; i.e., 
\begin{equation*}
\begin{split}
    \dsT^\pi(x)& \triangleq \inf\{t\geq 0|\phi(t;x,\pi)\in \XX_{\mathrm{goal}},\\
& \phi(\tau;x,\pi)\in \XX_{\mathrm{free}},\forall\tau\in [0,t]\}.
\end{split}
\end{equation*}
The problem of interest in this paper is to find the control policy $\pi^*:\XX\to\UU$ for system~\eqref{eqn:systemeqn} that incurs the minimal traveling time $\dsT^*$ for every $x\in\XX$; that is,
\begin{align*}
    \dsT^*(x)=\min_{\pi}\dsT^\pi(x), \forall x\in\XX
\end{align*}
and the minimizer $\pi^*$ is referred to as the optimal policy.

Notice that the minimal traveling time function $\dsT^*$ is not necessarily bounded, as $\XX_\mathrm{goal}$ may not be reachable for every $x\in\XX_\mathrm{free}$, and $\dsT^*(x)$ would be positive infinity for such states.
The positive infinity can cause difficulties in both computational expressions and mathematical analysis, and we address this issue by leveraging the Kruzhkov transform $\Psi:\real_{\geq0}\bigcup\{+\infty\}\to[0,1]$, defined as
\begin{align*}
    (\Psi\circ\dsT)(x)\triangleq 1-\exp(-\dsT(x)),
\end{align*}
and we let the transformed value function be denoted by $\varTheta\triangleq\Psi\circ\dsT$ inheriting all subscripts and superscripts of $\dsT$.
Notice that the Kruzhkov transform is bijective and monotonically increasing.
The positive infity is transformed to $1$ after the Kruzhkov transform while $0$ remains the same.
The objective of this paper is equivalently to find the transformed minimal traveling time function $\varTheta^*$ and the optimal policy $\pi^*$.

\section{The Incremental Policy (iPolicy) Algorithm}\label{sec:Algorithm}
In this section, we present the incremental policy algorithm iPolicy.
We leverage sampling-based methods in \cite{KF11,KF13,SJP14} and set-valued methods in \cite{cardaliaguet1999set,zhao2020pareto} along with asynchronous value iterations to incrementally approximate the minimal traveling time function.
In particular, iPolicy consists of two components: graph expansion and minimal traveling time estimation.
In graph expansion, iPolicy continually samples the free region $\XX_{\mathrm{free}}$ and constructs a graph using the set-valued method to approximate the continuous time dynamic of system \eqref{eqn:systemeqn}; in minimal traveling time estimation, asynchronous value iterations are executed over the approximate graph in a back propogation fashion, where only stale values are updated to save computations, to obtain an estimate of the minimal traveling time function.
As more samples are added, iPolicy iteratively expands graph and estimates the minimal traveling time function until a certain limit (e.g., time limit, approximation error threshold) is reached.
Denote the set of sampled states from $\XX_{\mathrm{free}}$ after $k$ iterations of iPolicy by $V_k$.
With slight abuse of notations, denote the estimate of minimal traveling time, also known as the value function, after $k$ iterations by $\dsT_k: V_k\to[0, +\infty]$.
The staleness function $\FF_k:V_k\to\natzero$ is defined as the number of iterations since the last update of a sample $x\in V_k$.
We proceed to explain iPolicy in the rest of this section.

\begin{algorithm}[h] \small
    \textit{Initialization}\\
    \For{$x\in V_0$}{\label{alg:iFPA:init:begin}
        $\varTheta_0(x)\leftarrow0$\;
        $\FF_0(x)\leftarrow P$\label{alg:iFPA:initializeStaleness}\;
    }\label{alg:iFPA:init:end}
    $k\leftarrow1$\;
    \textit{Main Loop}\label{alg:iFPA:mainLoopStart}\\
	\While {$k < K$}{
        $x_{\mathrm{new}}\leftarrow\texttt{Sample}(\XX_{\mathrm{free}}+d_k\BB)$\;\label{alg:iFPA:sampling}
        $V_k\leftarrow V_{k-1}\cup\{x_{\mathrm{new}}\}$\;\label{alg:iFPA:expansion}
        \If{$x_{\mathrm{new}}\in\XX_\mathrm{goal}+(M\epsilon_k+d_k)\BB$\label{alg:iFPA:mainloop:init:begin}}{
            $\hat{\varTheta}_k(x_{\mathrm{new}})\leftarrow0$\label{alg:iFPA:mainloop:init:value0}\;
        }\Else{
            $\hat{\varTheta}_k(x_{\mathrm{new}})\leftarrow1$\label{alg:iFPA:mainloop:init:value1}\;
        }
        $\hat{\FF}_k(x_{\mathrm{new}})\leftarrow P$\;\label{alg:iFPA:flagged}
        \For{$x\in V_{k-1}$}{
            $\hat{\varTheta}_k(x)\leftarrow \varTheta_{k-1}(x)$\;
            $\hat{\FF}_k(x)\leftarrow\FF_{k-1}(x)$\;
            Compute $F_k(x)$ in \eqref{eq:onehops}\;
        }\label{alg:iFPA:mainloop:init:end}
        $(\varTheta_k,\FF_k)\leftarrow {\texttt{ValueIteration}} {(V_k, \hat{\varTheta}_k,\hat{\FF}_k, m_k)}$\label{alg:iFPA:VI}\;
        $k\leftarrow{k+1}$;
	}
    \Return{$V_k$}\;
	\caption{Incremental Policy (iPolicy) Algorithm}
	\label{algorithm:iFPA}
\end{algorithm}

\SetKwProg{Def}{def}{:}{}
\begin{algorithm}[h] \small
    \caption{The \texttt{ValueIteration} procedure}\label{alg:VI}
	\textbf{Input:} Samples $V$, transformed value function $\varTheta$, staleness function $\FF$, recursion allowance $m$\;
    \Def{\textup{\texttt{ValueIteration}($V, \varTheta, \FF, m$)}}{
        $\Delta\leftarrow\Psi(\epsilon-d), \beta=1-\Delta$\;
        $S\leftarrow\{x\in V|\FF(x)=P, x\in\XX_{\mathrm{free}}+d\BB\}$\label{alg:VI:getStalenessSet}\;
        \For{$x\in S$}{
            $\varTheta(x)\leftarrow\texttt{BackProp}(x, m, \varTheta)$\;\label{alg:VI:backprop}
            $\FF(x)\leftarrow0$\;\label{alg:VI:clearStaleness}
        }
        \For{$x\in V\setminus S$}{
            $\FF(x)\leftarrow\FF(x)+1$\;\label{alg:VI:increaseStaleness}
        }
        \Return{$(\varTheta, \FF)$}
    }
\end{algorithm}	

\begin{algorithm}[h!] \small
    \caption{The \texttt{BackProp} procedure}\label{alg:backprop}
    \textbf{Input:} State $x\in V$, recursion allowance $m$, transformed value function $\varTheta^m$\;
    \Def{\textup{\texttt{BackProp}}\textup{($x$, $m$, $\varTheta^m$)}}{
        \If{$x\in\XX_{\mathrm{goal}}+(M\epsilon+d)\BB$ \textbf{or} $m=0$\label{alg:backprop:boundary}}{
        \Return $\varTheta^m(x)$\;
        }
        \For{$x'\in F(x)$\label{alg:backprop:onehops}}{
            $\varTheta^{m-1}(x)\leftarrow\texttt{BackProp}(x', m-1, \varTheta^m)$\;\label{alg:backprop:recursion}
        }
        $\displaystyle \varTheta^m(x)\leftarrow\Delta+\beta\min_{x'\in F(x)}\varTheta^{m-1}(x')$\;\label{alg:backprop:bellman}
        \Return $\varTheta^m(x)$
    }
\end{algorithm}

\subsection{Algorithm statement}
We initialize iPolicy with a set of states $V_0$, which we assume it intersects with the goal set $\XX_{\mathrm{goal}}$ to ensure the availability of boundary condition in value iteration execution.
Each sampled state $x\in V_0$ is initialized with its value set by $\varTheta_0(x)=0$ and the its staleness by a staleness threshold $P\geq0$, meaning their values shall be updated at the beginning of the main loop of iPolicy.
See lines \ref{alg:iFPA:init:begin}-\ref{alg:iFPA:init:end} in Algorithm~\ref{algorithm:iFPA}.

Then iPolicy enters the main loop in line~\ref{alg:iFPA:mainLoopStart}.
Similar to exploration, the main loop samples $\XX_{\mathrm{new}}$ and add newly sampled state $x_{\mathrm{new}}$ to $V_k$ as line~\ref{alg:iFPA:sampling}. 
Then the value and staleness of $x_{\mathrm{new}}$ are initialized in the same way as those in the initialization stage, while for sampled states from the last grid $V_{k-1}$, we retain their staleness and values as $\FF_{k-1}$ and $\varTheta_{k-1}$.
Since the new sample may incur paths with lower cost, value iterations shall be executed on $V_k$ as line~\ref{alg:iFPA:VI} in Algorithm~\ref{algorithm:iFPA} and Algorithm~\ref{alg:VI} to obtain a better estimate of the minimal traveling time $\dsT_k$.
The conventional value iteration is executed in a synchronous fashion, where the value at every sampled state on $V_k$ is updated in every iteration, and its computations grow quickly as more and more samples are added to $V_k$.
We employ the asynchronous value iteration to mitigate the computational complexity.
Particularly, the value at $x$ is updated when two circumstances occur: first, a sampled state is newly added to $V_k$, since its initial value may not be accurate; second, the value of a sampled state is too stale to reflect the true estimate of the minimal traveling time.
Recall the staleness function $\FF_k:V_k\to\natzero$ represents the number of iterations since the last update of a sample $x\in V_k$.
Let $\hat{\FF}_k$ be the staleness function before the update at iteration $k$.
Then $\hat{\FF}_k(x)=\FF_{k-1}(x), \forall x\in V_{k-1}$ and for $x\in V_k\setminus V_{k-1}$, we have $\hat{\FF}_k(x)=P$.
The staleness threshold $P\in\natzero$ is the maximum number of value iterations a sampled state can skip.
When $P=0$, the asynchronous update is identical to the synchronous update, while when $P$ is positive infinity, each sample will only be updated once.
In Algorithm~\ref{alg:VI} \texttt{ValueItertion}, the set of stale samples $S_k\triangleq\{x\in V_k|\hat{\FF}_k(x)=P, x\in\XX_{\mathrm{free}}+d_k\BB\}$ are updated using Algorithm~\ref{alg:backprop} \texttt{BackProp} that recursively searches for the minimum values from one's neighbors.
After updates, the staleness for the samples in $S_k$ will be cleared as line~\ref{alg:VI:clearStaleness} in Algorithm~\ref{alg:VI} as they have been freshly updated while staleness of other samples $V_k\setminus S_k$ would be increased by $1$ as line~\ref{alg:VI:increaseStaleness}.

Algorithm~\ref{alg:backprop} \texttt{BackProp} leverages the Bellman equation to recursively update values and partially solve for the estimate of minimal traveling time on $V_k$.
For simplicity, we drop the dependency on $k$ in the psuedo codes of \texttt{BackProp}.
The vanilla update procedure of value iteration is summarized as the Bellman operator $\TT_k$:
\begin{equation}\label{eqn:BellmanOperator}
(\TT_k\circ\dsT_k)(x)\triangleq\begin{cases}
    &\displaystyle\epsilon_k-d_k+\min_{x'\in F_k(x)}\dsT_k(x'), \\
        &\quad\text{if }x\in V_k\setminus(\XX_{\mathrm{goal}}+(M\epsilon_k+d_k)\BB);\\
    &\dsT_k(x), \text{otherwise},
\end{cases}
\end{equation}
where 
\begin{equation}\label{eq:onehops}
    \begin{aligned}
    F_k(x)\triangleq&(x+\epsilon_k\bigcup_{u\in\UU}f(x, u)+\rho(d_k, \epsilon_k)\BB)\\
    &\cap(\XX_{\mathrm{free}}+d_k\BB)\cap V_k
    \end{aligned}
\end{equation} 
is the set of one-hop neighbors of $x$ using Euler discretiztion and $\rho(d, k)\triangleq2d_k+l\epsilon_k(d_k+M\epsilon_k)$ is the perturbation radius inherited from \cite{cardaliaguet1999set}.
See Figure~\ref{fig:insights} for the visualization of the discretization, where the robot transits from $x$ to $x'$ after applying a constant control $u$ for a fixed time $\epsilon_k$, and we treat all sampled states (orange dots) in the circle of radius $\rho$ centered at $x'$ as the one-hop neighbors of $x$, as $x'$ is not necessarily in $V_k$.
With the Kruzhkov transform, one may rewrite the update procedure \eqref{eqn:BellmanOperator} and define a tranformed Bellman operator as:
\begin{align*}
    (\tT_k\circ\varTheta_k)(x)\triangleq
    \begin{cases}
        &\displaystyle\Delta_k+\beta_k\min_{x'\in F_k(x)}\varTheta_k(x'), \\
        &\quad\text{if }x\in V_k\setminus(\XX_{\mathrm{goal}}+(M\epsilon_k+d_k)\BB);\\
        &\varTheta_k(x), \text{ otherwise},
    \end{cases}
\end{align*}
where $\Delta_k\triangleq\Psi(\epsilon_k-d_k)$ is the transformed running cost, $\beta_k\triangleq1-\Delta_k$ is the total discount factor in the transformed form and $\tT_k$ is the transformed Bellman operator.
For each $x\in S$, \texttt{BackProp} computes $\varTheta^{m_k}(x)$ in a recursive manner.
See Figure~\ref{fig:insights} again for the visualization of \texttt{BackProp}.
Algorithm~\ref{alg:backprop} \texttt{BackProp} first checks whether the goal region $\XX_{\mathrm{goal}}$ is reached or the recursion allowance $m_k=0$ runs out as line~\ref{alg:backprop:boundary}, which are the boundary condition of Bellman equation and the limit of computations for each iteration, respectively.
If neither is satisfied, \texttt{BackProp} will check and compute $\varTheta^{m_k-1}(x')$ for each neighbor $x'\in F_k(x)$ by calling itself as line~\ref{alg:backprop:recursion}.
This will temporarily halt the exeuction of current \texttt{BackProp} (recursion allowance $m_k$) with all values saved, while processor will create a duplicate of \texttt{BackProp} with a reduced recursion allowance $m_k-1$ and continues executing the new \texttt{BackProp}
When the new \texttt{BackProp} finishes according to the boundary conditions in line~\ref{alg:backprop:boundary}, the processor will exit it and resume the execution of the old \texttt{BackProp} with recursion allowance $m_k$ with its previously saved values and the returned value at a one-hop neighbor $\varTheta^{m_k-1}(x')$ from the new \texttt{BackProp}.
With all values at one-hop neighbors being ready, \texttt{BackProp} computes $\varTheta^m(x)$ using the Bellman operator as line~\ref{alg:backprop:bellman}.
In other words, a recursion with maximum allowance $m_k$ starting from $x$ updates the value at $x$ by propogating values at neighbors of $x$ within $m_k$ hops.
This limits the search to a subset of $V_k$ to save computations.
When $m_k\geq|V_k|$, \texttt{BackProp} updates all sampled states on $V_k$ in the depth-first search manner.

\begin{figure}[t]
    \includegraphics[width=0.48\textwidth]{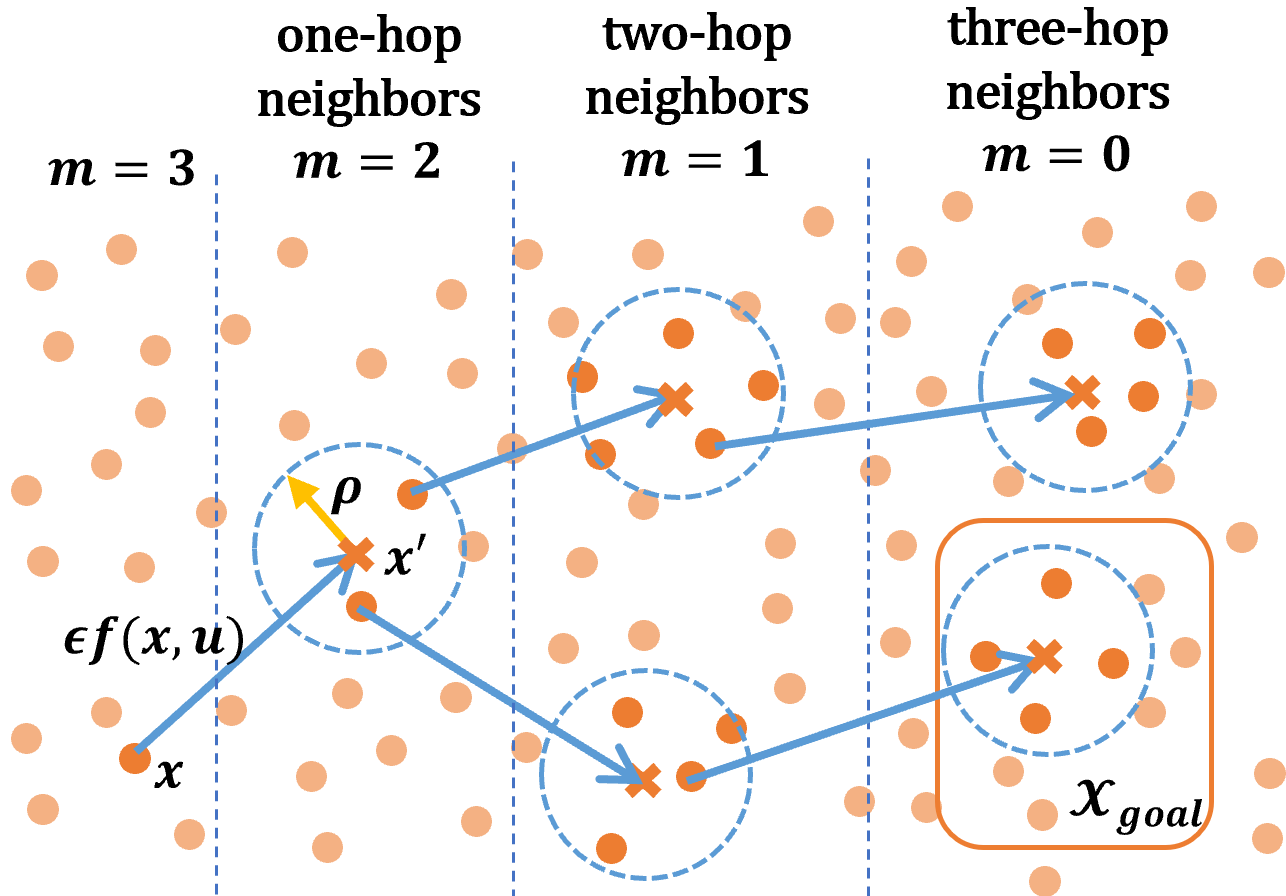}
    \caption{Illustration of the approximation of dynamic and \texttt{BackProp}. Orange dots are sampled states, orange crosses are resting states when applying constant control $u$ for time $\epsilon$ and blue arrows imply discrete time transition.}\label{fig:insights}
\end{figure}

The consistency of iPolicy follows the discretization scheme in \cite{cardaliaguet1999set}, where convergence rate conditions are imposed for the temporal and spatial resolutions.
Specifically, let $d_k$ be a conservative estimate of the spatial dispersion, $\epsilon_k$ be the associated temporal resolution and $\rho(\epsilon_k, d_k)$ be the perturbation radius.
Then $d_k$ is a distance within which every point $x\in\XX$ can find a sampled state in $V_k$ and $\epsilon_k$ would be the minimal time elapse system \eqref{eqn:systemeqn} can transit.
Since system \eqref{eqn:systemeqn} may not fall on a sample in $V_k$ if it takes a constant control for $\epsilon_k$ time units, the discrepancy of approximation of the dynamic system is accommodated by the perturbation $\rho(d_k, \epsilon_k)$ so that the one-hop neighbors of state $x$ with pertubation would surely fall on $V_k$; i.e., $(x+\epsilon_kf(x, u)+\rho(d_k, \epsilon_k)\BB)\cap V_k\neq\emptyset$.
The choice of $d_k$ and $\epsilon_k$ is not unique but is subject to the following assumption for the purpose of consistent approximation.
\begin{assumption}\label{asmp:resolutions}
    The selection of $d_k$ and $\epsilon_k$ satsifies the following relations:
    \begin{enumerate}[label=\textbf{(A\arabic*)}, leftmargin = *, align=left]
        \setcounter{enumi}{4}
        \item \label{asmp:resolutions:spatial} $d_k\geq B\Big(\frac{\log|V_k|}{|V_k|}\Big)^{1/n}, \forall k\in\natzero$, where $B>(\mu(\XX)/C_n)^{1/n}$;
        \item \label{asmp:resolutions:convergence} $\epsilon_k$ and $d_k/\epsilon_k$ monotonically decreases to $0$ as $k\to+\infty$;
        \item \label{asmp:resolutions:epsilonLarger} $\epsilon_k>d_k, \forall k\in\natzero$;
        \item \label{asmp:resolutions:rho}$\rho(d_k, \epsilon_k)\geq d_{k-1}, \forall k\in\natural$.
    \end{enumerate}
\end{assumption}
Assumption~\ref{asmp:resolutions:spatial} follows from \cite{KF11} and implies the lower bound of the spatial resolution that the union of balls centered at every sample $x\in V_k$ with radius $d_k$ covers the whole region $\XX$. 
Assumption~\ref{asmp:resolutions:convergence} characterizes the convergence rate for temporal and spatial resolutions.
Notice that the perturbation radius $\rho(d_k, \epsilon_k)$ is dimishing faster than $\epsilon_k$; this ensures the discrete time transition always falls on the graph for sufficiently small temporal-spatial resolutions.
In Assumption~\ref{asmp:resolutions:convergence}, the faster convergence of $d_k$ compared to $\epsilon_k$ indicates the discrete time transition on $V_k$ would be close to the discrete time transition on $\XX$ for sufficiently small $\epsilon_k$; and $\epsilon_k\to0$ finalizes the consistency of approximation.
Assumptions~\ref{asmp:resolutions:epsilonLarger} and \ref{asmp:resolutions:rho} are purely for the validity of computations and analysis of iPolicy.
Assumption~\ref{asmp:resolutions:epsilonLarger} trivially holds for sufficiently large $k$ by Assumption~\ref{asmp:resolutions:convergence} while assumption~\ref{asmp:resolutions:rho} holds when $d_k$ takes the lower bound in Assumption~\ref{asmp:resolutions:spatial}.

\subsection{Performance guarantee}
The consistency of the estimates requires proper selection of temporal resolution $\epsilon_k$ and sufficient value iterations ${m_k}$ executed on $V_k$.
The following assumption is adapted from Assumption IV.1 in \cite{zhao2020pareto} and it imposes the minimum required value iterations $\{m_k\}$ executed on each graph with convergence guarantee.
\begin{assumption}\label{asmp:discountWindow}
    There exists a graph index $k_0$, an upper bound of intervals $T\geq0$ and an upper bound of accumulated discount $\bar{\beta}\in(0, 1)$ such that $\forall\bar{k}\geq k_0$, the total numbers of value iterations on each graph $\{m_k\}$ satisfy
    \begin{align*}
        \prod_{t=0}^T\max_{\ell\in\{0, \dots, P\}}\beta_{\breve{k}(t, \ell)}^{m_{\breve{k}(t, \ell)}}\leq\bar{\beta}<1,
    \end{align*}
    where the graph index is denoted by $\breve{k}(t, \ell)\triangleq \bar{k}+t(P+1)+\ell$.
\end{assumption}
In particular, Assumption~\ref{asmp:discountWindow} partitions the graph indices into intervals of equal length $P+1$, and the aggregation of the worst-case discount on each interval $\breve{k}\in[\bar{k}+t(P+1), \bar{k}+t(P+1)+P]$ over $T+1$ consecutive intervals should be strictly lower than $1$.
This implies the distance between the estimated value function $\varTheta_k^{m_k}$ and the fixed point $\varTheta_k^*$ is strictly decreasing in a window of $(T+1)(P+1)$ graphs.

The main result is summarized below, where the last estimate of minimal traveling time $\varTheta^{m_k}_k$ converges to the ground truth pointwise with probability one.
\begin{theorem}\label{theorem:mainResult}
    Suppose Assumptions~\ref{asmp:system}, \ref{asmp:resolutions} and \ref{asmp:discountWindow} hold.
    Then the following holds with probability one:
    \begin{align*}
        \lim_{k\to+\infty}\min_{x'\in(x+d_k\BB)\cap V_k}\varTheta^{m_k}_k=\varTheta^*(x), \forall x\in\XX_{\mathrm{free}}\setminus\ZZ,
    \end{align*}
    where $\ZZ$ is a set of sufficiently small measure.
\end{theorem}

\section{Analysis}\label{sec:Analysis}

In this section, we analyze the performance of iPolicy.
Towards the proof of Theorem~\ref{theorem:mainResult}, we proceed with the following result with a supplementary definition.
\begin{theorem}\label{theorem:contractiveOverPeriods}
    Suppose Assumptions~\ref{asmp:system}, \ref{asmp:resolutions} and \ref{asmp:discountWindow} hold and let $\varTheta_k^*(x)=\varTheta_k^m(x)=1, \forall x\in\XX\setminus(\XX_{\mathrm{free}}+d_k\BB), k\geq1$ and $0\leq m\leq m_k$.
    Then it holds with probability one:
    \begin{align*}
        \lim_{k\to+\infty}\min_{x'\in(x+d_k\BB)\cap V_k}\varTheta^{m_k}_k=\varTheta^*(x), \forall x\in\XX\setminus\ZZ,
    \end{align*}
    where $\ZZ$ is a set of sufficiently small measure.
\end{theorem}
\begin{remark}
Theorem~\ref{theorem:contractiveOverPeriods} extends Theorem~\ref{theorem:mainResult} to the whole state space $\XX$ by manually setting $\varTheta_k^{m_k}(x)$ as $1$ for every $x\in\XX_{\mathrm{obs}}$, indicating there is no feasible path connecting the goal set $\XX_{\mathrm{goal}}$. 
This practice consolidates the obstacle regions $\XX_{\mathrm{obs}}$ where $\varTheta^m_k$ is only defined with the free regions $\XX_{\mathrm{free}}$ and simplifies analysis by reducing the total number of cases considered in the analysis.
\end{remark}
The following lemma shows the supplementary definition in Theorem~\ref{theorem:contractiveOverPeriods} is not involved in any part in the execution of the proposed algorithms.
\begin{lemma}\label{lemma:supplementaryDefinition}
    For any $k\geq0$ and $x\in\XX\setminus(\XX_{\mathrm{free}}+d_k\BB)$, it holds that $x\notin S_\kappa$ and $x\notin F_\kappa(x')$, $\forall\kappa\geq k$ and $x'\in V_k\setminus\{x\}$. 
\end{lemma}
\begin{proof}
It follows from Assumption~\ref{asmp:resolutions} that $d_k$ is monotonically decreasing.
Then for every $x\in\XX\setminus(\XX_{\mathrm{free}}+d_k\BB)$, $x\notin\XX_{\mathrm{free}}+d_\kappa\BB, \forall\kappa\geq k$.
This implies $x$ will never be included in $S_\kappa, \forall \kappa\geq k$ per line~\ref{alg:VI:getStalenessSet} in Algorithm~\ref{alg:VI}.
It follows from the definition of $F_k(x)$ in \eqref{eq:onehops} that $x\notin F_\kappa(x'), \forall x'\in V_\kappa$ and $\kappa\geq k$.
This completes the proof.
\end{proof}
Lemma~\ref{lemma:supplementaryDefinition} shows that the supplementary definition of $\varTheta_k^m$ and $\varTheta_k^*$ over $\XX\setminus(\XX_{\mathrm{free}}+d_k\BB)$ will not affect computations of $\varTheta_k^m$ and $\varTheta_k^*$ over the free region, and Theorem~\ref{theorem:contractiveOverPeriods} is identical to Theorem~\ref{theorem:mainResult} over $\XX_{\mathrm{free}}+d_k\BB$.
Throughout the analysis in this section, we assume all conditions in Theorem~\ref{theorem:contractiveOverPeriods} hold.

\subsection{Preliminaries}

In this subsection, we introduce a couple of preliminary results that facilicate the proof of our main theorem.
The following lemma is borrowed from Lemma VII.4 in \cite{zhao2020pareto} and is listed below for the completeness of the paper.
\begin{lemma}[Lemma VII.4 in \cite{zhao2020pareto}]\label{lemma:perturbedNormInequality}
    Given $\XX_1, \XX_2\subseteq\XX$, consider $V_i:\XX\to [0, 1]$ and $\eta_i>0$ s.t. $(x+\eta_i\BB)\cap\XX_i\neq\emptyset, \forall x\in\XX$, where $i\in\{1, 2\}$.
    Then for any set-valued map $Y:\XX\rightrightarrows\XX$ s.t. $Y(x)\neq\emptyset, \forall x\in\XX$, we have
    \begin{equation}
        \begin{aligned}\label{eq:perturbed2perturbed}
            &\|\min_{\tilde{x}\in(Y(x)+\eta_1\BB)\cap\XX_1}V_1(\tilde{x})-\min_{\tilde{x}\in(Y(x)+\eta_2\BB)\cap\XX_2}V_2(\tilde{x})\|_\XX\\
            \leq&\|\min_{\tilde{x}\in(x+\eta_1\BB)\cap\XX_1}V_1(\tilde{x})-\min_{\tilde{x}\in(x+\eta_2\BB)\cap\XX_2}V_2(\tilde{x})\|_\XX
        \end{aligned}
    \end{equation}
    If $\XX_1=\XX_2\triangleq\bar{\XX}$ and $\eta_1=\eta_2\triangleq\eta$, then
    \begin{align}\label{eq:perturbed2exact}
        \|\min_{\tilde{x}\in(x+\eta\BB)\cap\bar{\XX}}V_1(\tilde{x})-\min_{\tilde{x}\in(x+\eta\BB)\cap\bar{\XX}}V_2(\tilde{x})\|_{\XX}\leq\|V_1-V_2\|_{\bar{\XX}}.
    \end{align}
\end{lemma}
The following lemma shows that the maximum over a finite horizon of a convergent sequence is also convergent.
\begin{lemma}\label{lemma:maxConvergence}
    Given a sequence $\{a_k\}$ s.t. $\lim_{k\to+\infty}a_k=0$, consider a new sequence $\{b_k\}$ s.t. $b_k=\max_{\ell\in\II}a_{k+\ell}$, where $\II\subseteq\natzero$ is finite set of integers.
    Then $\lim_{k\to+\infty}b_k=0$.
\end{lemma}
\begin{proof}
    Since $\lim_{k\to+\infty}a_k=0$, then $\forall\epsilon>0$, $\exists K>0$ s.t. $\forall k\geq K$, $|a_k|\leq\epsilon$.
    Then it holds that $|b_k|=|\max_{\ell\in\II}a_{k+\ell}|\leq\max_{\ell\in\II}|a_{k+\ell}|\leq\epsilon$.
    This implies $\lim_{k\to+\infty}b_k=0$.
\end{proof}

The following lemma from Lemma~7.2 in \cite{MZKF13} characterizes the upper bound of the spatial resolution $d_k$.
\begin{lemma}\label{lemma:dispersion}
    Consider an estimate of spatial dispersion $d\geq B(\frac{\log|V|}{|V|})^{1/n}$.
    Then it holds that 
    \begin{align*}
        \limsup_{|V|\to+\infty}\mathbb{P}[\exists x\in\XX\text{ s.t. }(x+d\BB)\cap V=\emptyset]=0.
    \end{align*}
\end{lemma}
\begin{remark}
    Lemma~\ref{lemma:dispersion} indicates that for any $x\in\XX$, there always is at lease one node in $V_k$ within the spatial dispersion $d_k$ of $x$ with probability one in an asymptotic sense.
    Via the spatial dispersion, $V_k$ can be treated as a grid in \cite{cardaliaguet1999set} and we can leverage all set-valued tools for analysis.
\end{remark}

The following lemma shows the whole graph will be periodically updated within an interval of length $P$.
\begin{lemma}\label{lemma:iterationIntervalInclusion}
    It holds that $V_{k+P}\subseteq\bigcup_{\ell=k}^{k+P}S_\ell$.
\end{lemma}
\begin{proof}
    Fix $x\in V^{k+P}$ and three cases arise:
    \begin{itemize}
        \item Case 1: $x\in S_k$. Then the lemma trivially holds;
        \item Case 2: $x\in V_k\setminus S_k$. 
        This implies $\hat{\FF}_k(x)<P$. 
        Then the staleness of $x$ increases to $P$ at iteration $P-\hat{\FF}_k(x)+k$.
        With that being said, $x\in S_{P-\hat{\FF}_k(x)+k}\subseteq\bigcup_{\ell=k}^{k+P}S_\ell$ and this completes the proof in Case 2;
        \item Case 3: $x\in V_{k+P}\setminus V_k$. 
        Then $x$ is a newly added node at some iteration $k'\in[k+1, k+P]$ and $x\in S_{k'}\bigcup_{\ell=k}^{k+P}S_\ell$. 
        This completes the proof in Case 3.
    \end{itemize}
    Since the above holds $\forall x\in V^{k+P}$, proof is completed.
\end{proof}

The following lemma characterizes the values of nodes within goal regions.
\begin{lemma}\label{lemma:goalNodes}
    Given $k\geq1$ and any $x\in\XX+(M\epsilon_k+d_k)\BB$, it holds that $\varTheta_k^\ell(x)=\varTheta_k^*(x)=0$.
\end{lemma}
\begin{proof}
    Following the definition of the Bellman operator $\tT_k$ and the initial values of $\varTheta_0$, one sees that $\varTheta_k^\ell(x)=\varTheta_k^{\ell-1}(x)=\cdots=\varTheta_k^0(x)=0$.
    The same holds for $\varTheta_k^*(x)$.
    This completes the proof.    
\end{proof}

\subsection{Contraction property on a fixed graph}
In this subsection, we analyze the monotonicity of the Bellman operator $\tT_k$ on a fixed graph, which is summarized in the following lemma.
\begin{lemma}\label{lemma:contractiveOverVI}
    The distance between a value function $\varTheta_k$ and the fixed point $\varTheta^*_k$ is monotonically decreasing over the Bellman operator $\tT_k$; i.e., 
    \begin{align}\label{eq:contractiveOverVI:monotone}
        \|\tT_k(\varTheta_k)-\varTheta_k^*\|_{V_k}\leq\|\varTheta_k-\varTheta_k^*\|_{V_k}.
    \end{align}
    Specially, strict monotonic decrease holds $\forall x\in S_k$; that is, 
    \begin{align}\label{eq:contractiveOverVI:discount}
        |\tT_k(\varTheta_k)(x)-\varTheta_k^*(x)|\leq\beta_k\|\varTheta_k-\varTheta_k^*\|_{F_k(x)}.
    \end{align}
    Moreover, the perturbed version of the above special case holds with probablity one in the asymptotic sense:
    \begin{equation}\label{eq:contractiveOverVI:discount:perturbed}
        \begin{aligned}
            &|\tT_k(\varTheta_k)(x)-\varTheta_k^*(x)|\\
            \leq&\beta_k\|\min_{x'\in(x+d_k\BB)\cap V_k}\varTheta_k(x')-\min_{x'\in(x+d_k\BB)\cap V_k}\varTheta^*_k(x')\|_{\XX\setminus\ZZ}\\
            \leq&\beta_k\|\varTheta_k-\varTheta_k^*\|_{V_k},
        \end{aligned}
    \end{equation}
    where $\ZZ$ is a set of sufficiently small measure.
\end{lemma}
\begin{proof}
    We first proceed to show the proofs of \eqref{eq:contractiveOverVI:discount} and \eqref{eq:contractiveOverVI:discount:perturbed}.
    Once \eqref{eq:contractiveOverVI:discount} and \eqref{eq:contractiveOverVI:discount:perturbed} are proven, we extend the results to \eqref{eq:contractiveOverVI:monotone}.

    \begin{claim}\label{claim:contractiveOverVI:discount}
        Inequality \eqref{eq:contractiveOverVI:discount} holds for every $x\in S_k$.
    \end{claim}
    \begin{proof}
    Notice that $\tT_k(\varTheta^*_k)=\varTheta^*_k$, since $\varTheta^*_k$ is the fixed point of $\tT_k$.
    Two cases arise.
    \begin{itemize}
        \item Case 1, $x\in\XX_\mathrm{goal}+(M\epsilon_k+d_k)\BB$. Then it follows from Lemma~\ref{lemma:goalNodes} that $|\tT_k(\varTheta_k)(x)-\tT_k(\varTheta^*_k)(x)|=0\leq\beta_k\|\varTheta_k-\varTheta_k^*\|_{V_k}$.
        \item Case 2, $x\in S_k\setminus(\XX_\mathrm{goal}+(M\epsilon_k+d_k)\BB)$. 
        Then the following inequality holds:
        \begin{equation}\label{eq:contractiveOverVI:1}
            \begin{aligned}
                &|\tT_k(\varTheta_k)(x)-\tT_k(\varTheta^*_k)(x)|\\
                =&\Big|\Delta_k+\beta_k\min_{x'\in F_k(x)}\varTheta(x')-\Delta_k-\beta_k\min_{x'\in F_k(x)}\varTheta^*_k(x')\Big|\\
                =&\beta_k\Big|\min_{x'\in F_k(x)}\varTheta(x')-\min_{x'\in F_k(x)}\varTheta^*_k(x')\Big|\\
                \leq&\beta_k\|\varTheta_k-\varTheta_k^*\|_{F_k(x)}.
            \end{aligned}
        \end{equation}
    \end{itemize}
    In summary, we have $\|\tT_k(\varTheta_k)-\tT_k(\varTheta^*_k)\|_{S_k}\leq\beta_k\|\varTheta_k-\varTheta_k^*\|_{V_k}$.
    This completes the proof of Claim~\ref{claim:contractiveOverVI:discount} and \eqref{eq:contractiveOverVI:discount} is proven.
    \end{proof}

    \begin{claim}\label{claim:contractiveOverVI:discount:perturbed}
        Inequality \eqref{eq:contractiveOverVI:discount:perturbed} holds with probability one in the asymptotic sense.
    \end{claim}
    \begin{proof}
    Following the argument towards Claim~\ref{claim:contractiveOverVI:discount}, Case 1 trivially holds for Claim~\ref{claim:contractiveOverVI:discount:perturbed}.
    We rewrite \eqref{eq:contractiveOverVI:1} for $x\in S_k\setminus(\XX_\mathrm{goal}+(M\epsilon_k+d_k)\BB)$ and arrive at
    \begin{equation}\label{eq:contractiveOverVI:discount:perturbed:1}
        \begin{aligned}
            &|\tT_k(\varTheta_k)(x)-\tT_k(\varTheta^*_k)(x)|\\
            =&\beta_k\Big|\min_{x'\in F_k(x)}\varTheta_k(x')-\min_{x'\in F_k(x)}\varTheta^*_k(x')\Big|\\
            \leq&\beta_k\Big\|\min_{x'\in F_k(x)}\varTheta_k(x')-\min_{x'\in F_k(x)}\varTheta^*_k(x')\Big\|_{V_k}.
        \end{aligned}
    \end{equation}
    It follows from the supplementary definition in Theorem~\ref{theorem:contractiveOverPeriods} that $\varTheta_k(x)=1, \forall x\in\XX\setminus(\XX_{\mathrm{free}}+d_k\BB)$.
    Then it follow from the definition of $F_k(x)$ in \eqref{eq:onehops} that 
    \begin{align*}
        \min_{x'\in F_k(x)}\varTheta_k(x')=\min_{x'\in(x+\epsilon_k\bigcup_{u\in U}f(x, u)+\rho(d_k, \epsilon_k)\BB)\cap V_k}\varTheta(x').
    \end{align*}
    Same holds for $\varTheta_k^*$.
    Then by Lemma~\ref{lemma:dispersion} and \eqref{eq:perturbed2perturbed} in Lemma~\ref{lemma:perturbedNormInequality}, \eqref{eq:contractiveOverVI:discount:perturbed:1} renders with probability one at
    \begin{equation}\label{eq:contractiveOverVI:discount:perturbed:2}
        \begin{aligned}
            &\|\min_{x'\in F_k(x)}\varTheta_k(x')-\min_{x'\in F_k(x)}\varTheta^*_k(x')\|_{V_k}\\
            \leq&\|\min_{x'\in(x+\rho(d_k, \epsilon_k)\BB)\cap V_k}\varTheta_k(x')\\
            &\quad-\min_{x'\in(x+\rho(d_k, \epsilon_k)\BB)\cap V_k}\varTheta^*_k(x')\|_{V_k}.
        \end{aligned}
    \end{equation}
    Since $\rho(d_k, \epsilon_k)\geq d_k$, we apply \eqref{eq:perturbed2perturbed} again and \eqref{eq:contractiveOverVI:discount:perturbed:2} renders at
    \begin{equation}\label{eq:contractiveOverVI:discount:perturbed:3}
        \begin{aligned}
            &\|\min_{x'\in(x+\rho(d_k, \epsilon_k)\BB)\cap V_k}\varTheta_k(x')\\
            &\quad-\min_{x'\in(x+\rho(d_k, \epsilon_k)\BB)\cap V_k}\varTheta^*_k(x')\|_{V_k}\\
            \leq&\|\min_{x'\in(x+d_k\BB)\cap V_k}\varTheta_k(x')-\min_{x'\in(x+d_k\BB)\cap V_k}\varTheta^*_k(x')\|_{V_k}.
        \end{aligned}
    \end{equation}
    Since $\ZZ$ is a set of sufficiently small measure, a new node falls into $\ZZ$ with sufficiently small probability; i.e., $\mathbb{P}[V_k\cap\ZZ\neq\emptyset]=0$.
    With that being said, it follows again from \eqref{eq:perturbed2perturbed} and \eqref{eq:contractiveOverVI:discount:perturbed:3} that the following holds with probability one:
    \begin{align*}
        &\|\min_{x'\in(x+d_k\BB)\cap V_k}\varTheta_k(x')-\min_{x'\in(x+d_k\BB)\cap V_k}\varTheta^*_k(x')\|_{V_k}\\
        \leq&\|\min_{x'\in(x+d_k\BB)\cap V_k}\varTheta_k(x')-\min_{x'\in(x+d_k\BB)\cap V_k}\varTheta^*_k(x')\|_{\XX\setminus\ZZ}.
    \end{align*}
    Then it follows from \eqref{eq:perturbed2exact}, the above inequality renders at 
    \begin{align*}
       &\|\min_{x'\in(x+d_k\BB)\cap V_k}\varTheta_k(x')-\min_{x'\in(x+d_k\BB)\cap V_k}\varTheta^*_k(x')\|_{\XX\setminus\ZZ}\allowdisplaybreaks\\
       \leq&\|\min_{x'\in(x+d_k\BB)\cap V_k}\varTheta_k(x')-\min_{x'\in(x+d_k\BB)\cap V_k}\varTheta^*_k(x')\|_{\XX}\allowdisplaybreaks\\
       \leq&\|\varTheta_k-\varTheta^*_k\|_{V_k}.
    \end{align*}
    Combining all above relations completes the proof.
    \end{proof}

    \begin{claim}\label{claim:contractiveOverVI:monotone}
        Inequality \eqref{eq:contractiveOverVI:monotone} holds true.
    \end{claim}
    \begin{proof}
        It follows from Lemma~\ref{lemma:goalNodes} that for $x\in S_k\setminus(\XX_\mathrm{goal}+(M\epsilon_k+d_k)\BB)$, $|\tT_k(\varTheta_k)(x)-\tT_k(\varTheta_k^*)(x)|=0\leq\|\varTheta_k-\varTheta_k^*\|_{V_k}$.
        Then it follows from Claim~\ref{claim:contractiveOverVI:discount} that \eqref{eq:contractiveOverVI:monotone} holds for $x\in S_k$.
        This completes the proof.
    \end{proof}
    It follows from Claims~\ref{claim:contractiveOverVI:discount}, \ref{claim:contractiveOverVI:discount:perturbed} and \ref{claim:contractiveOverVI:monotone} that Lemma~\ref{lemma:contractiveOverVI} holds. 
    This completes the proof.
\end{proof}

\begin{remark}
    The sufficiently small set $\ZZ$ is to accommodate the result of Corollary~\ref{corollary:uniformConvergence} in the proof of Theorem~\ref{theorem:mainResult}, where the uniform convergence of the Kruzhkov transformed functions holds in an almost sure sense.
\end{remark}

\subsection{Asynchronous contraction over graphs}
Following Algorithm~\ref{alg:backprop}, we recursively define the set of updated nodes by $S_k^{m_k}\triangleq S_k$, and $S_k^\ell=\bigcup_{x\in S_k^{\ell+1}}F_k(x), \forall\ell\in\{0, \dots, m_k-1\}$.
Correspondingly, define the value function by $\varTheta^\ell_k(x)\triangleq\tT_k(\varTheta^{\ell-1}_k)(x), \forall x\in S_k^{\ell}$ and $\ell\in\{1, \dots, m_k\}$, and $\varTheta^0_k(x)=\hat{\varTheta}_k(x), \forall x\in V_k$.
Starting from this subsection, we consider two perturbed estimates of the minimal traveling time function $\displaystyle\tilde{\varTheta}_k^*(x)\triangleq\min_{x'\in(x+d_k\BB)\cap V_k}\varTheta_k^*(x')$ and $\displaystyle\tilde{\varTheta}_k^{*, -}(x)\triangleq\min_{x'\in(x+d_{k-1}\BB)\cap V_k}\varTheta_k^*(x')$.
Notice that the perturbation radii $d_k$ and $d_{k-1}$ below the minimization operator are different.
Similar notations apply to $\tilde{\varTheta}_k^m$ and $\tilde{\varTheta}_k^{m, -}$ respectively for $m\in\natzero$.
The following lemma characterizes the contraction property between two consecutive graphs.
\begin{lemma}\label{lemma:errorAlignedOverGrid}
    Consider $x\in S_k$. It holds with probability one:
    \begin{align*}
        |\varTheta^{m_k}_k(x)-\varTheta^*_k(x)|\leq\beta_k^{m_k}\|\varTheta^{m_{k-1}}_{k-1}-\varTheta^*_{k-1}\|_{V_{k-1}}+\beta_k^{m_k}c_k,
    \end{align*}
    where $c_k\triangleq\|\tilde{\varTheta}^*_{k-1}-\tilde{\varTheta}^{*, -}_k\|_{\XX\setminus\ZZ}$ and $\ZZ$ is sufficiently small.
\end{lemma}
\begin{proof}
    Fix $x\in S_k$.
    By line~\ref{alg:VI:backprop} in Algorithm~\ref{alg:VI}, we apply \eqref{eq:contractiveOverVI:discount} in Lemma~\ref{lemma:contractiveOverVI} for $m_k-1$ times and it renders at 
    \begin{equation}\label{eq:errorAlignedOverGrid:1}
        \begin{aligned}
            |\varTheta^{m_k}_k(x)-\varTheta^*_k(x)|\leq\beta_k^{m_k-1}\|\varTheta^1_k-\varTheta^*_k\|_{V_k}.    
        \end{aligned}
    \end{equation}
    Then it again follows from \eqref{eq:contractiveOverVI:monotone}, \eqref{eq:contractiveOverVI:discount:perturbed} in Lemma~\ref{lemma:contractiveOverVI} and Assumption~\ref{asmp:resolutions:rho} that 
    \begin{equation}\label{eq:errorAlignedOverGrid:2}
        \begin{aligned}
            &\|\varTheta^1_k-\varTheta^*_k\|_{V_k}\\
            \leq&\beta_k\|\min_{x'\in(x+\rho(d_k, \epsilon_k)\BB)\cap V_k}\varTheta^0_k(x')\\          &-\min_{x'\in(x+\rho(d_k, \epsilon_k)\BB)\cap V_k}\varTheta^*_k(x')\|_{\XX\setminus\ZZ}\\
            \leq&\beta_k\|\min_{x'\in(x+d_{k-1}\BB)\cap V_k}\varTheta_k^0(x')-\min_{x'\in(x+d_{k-1}\BB)\cap V_k}\varTheta^*_k(x')\|_{\XX\setminus\ZZ}\\
            =&\beta_k\|\tilde{\varTheta}_k^{0, -}-\tilde{\varTheta}_k^{*,-}\|_{\XX\setminus\ZZ}.
        \end{aligned}
    \end{equation}
    The last inequality follows from \eqref{eq:perturbed2perturbed} in Lemma~\ref{lemma:perturbedNormInequality}.
    Let $x_k\in V_k\setminus V_{k-1}$ and two cases arise.
    We proceed to show both cases render at the same result, summarized in \eqref{eq:errorAlignedOverGrid:4}.

    Case 1, $x_k\in\XX\setminus(\XX_{\mathrm{goal}}+(M\epsilon_k+d_k)\BB)$.
    It follows from Lemma~\ref{lemma:goalNodes} that $\varTheta_k^0(x_k)=\varTheta_k^*(x_k)=0$; that is, $\tilde{\varTheta}_k^{0, -}(x)=\tilde{\varTheta}_k^{*, -}(x)=0, \forall x\in x_k+d_{k-1}\BB$.
    Then we have 
    \begin{equation}\label{eq:errorAlignedOverGrid:3:case1:1}
        \begin{aligned}
        &\|\tilde{\varTheta}_k^{0,-}-\tilde{\varTheta}_k^{*,-}\|_{\XX\setminus\ZZ}\\
        =&\|\tilde{\varTheta}_k^{0,-}-\tilde{\varTheta}_k^{*,-}\|_{\XX\setminus(\ZZ\cup(x_k+d_{k-1}\BB))}\allowdisplaybreaks\\
        \leq&\|\tilde{\varTheta}_k^{0, -}-\tilde{\varTheta}_{k-1}^*\|_{\XX\setminus(x_k+d_{k-1}\BB)}+c_k.
        \end{aligned}
    \end{equation}
    We focus on the first term.
    For $x\in\XX\setminus(\ZZ\cup(x_k+d_{k-1}\BB))$, since no changes are made to values at $x\in V_{k-1}$, we have
    \begin{equation}\label{eq:errorAlignedOverGrid:3:case1:2}
        \begin{aligned}
            &\tilde{\varTheta}^{0, -}_k(x)=\min_{x'\in(x+d_{k-1}\BB)\cap V_k}\varTheta^0_k(x')\allowdisplaybreaks\\
            =&\min_{x'\in(x+d_{k-1}\BB)\cap V_{k-1}}\varTheta^{m_{k-1}}_{k-1}(x')=\tilde{\varTheta}^{m_{k-1}}_{k-1}(x).    
        \end{aligned}
    \end{equation} 
    Combining \eqref{eq:errorAlignedOverGrid:3:case1:1} and \eqref{eq:errorAlignedOverGrid:3:case1:2} renders at
    \begin{equation}\label{eq:errorAlignedOverGrid:3:case1:3}
        \begin{aligned}
            \|\tilde{\varTheta}_k^{0, -}-\tilde{\varTheta}_k^{*,-}\|_{\XX\setminus\ZZ}\leq&\|\tilde{\varTheta}^{m_{k-1}}_{k-1}-\tilde{\varTheta}^*_{k-1}\|_{\XX\setminus(x_k+d_{k-1}\BB)}+c_k\\
            \leq&\|\tilde{\varTheta}^{m_{k-1}}_{k-1}-\tilde{\varTheta}^*_{k-1}\|_\XX+c_k.
        \end{aligned}
    \end{equation}

    Case 2, $x_k\in\XX_{\mathrm{goal}}+(M\epsilon_k+d_k)\BB$.
    Then by the triangular inequality of supremum norm, one has a result similar to \eqref{eq:errorAlignedOverGrid:3:case1:1}:
    \begin{equation}\label{eq:errorAlignedOverGrid:3:case2:1}
        \begin{aligned}
            \|\tilde{\varTheta}_k^{0,-}-\tilde{\varTheta}_k^{*,-}\|_{\XX\setminus\ZZ}\leq\|\tilde{\varTheta}_k^{0, -}-\tilde{\varTheta}_{k-1}^*\|_{\XX}+c_k.
        \end{aligned}
    \end{equation}
    We focus on the first term $\|\tilde{\varTheta}^{0, -}_k-\tilde{\varTheta}^*_{k-1}\|_{\XX}$. 
    For any $x\in\XX\setminus(x_k+d_{k-1}\BB)$, it renders at the identical result as \eqref{eq:errorAlignedOverGrid:3:case1:2}.
    For any $x\in x_k+d_{k-1}\BB$, since $x_k\in\XX\setminus(\XX_{\mathrm{goal}}+(M\epsilon_k+d_k)\BB)$, it follows line~\ref{alg:iFPA:mainloop:init:value1} in Algorithm~\ref{algorithm:iFPA} that $\min_{x'\in(x+d_{k-1}\BB)\cap V_k}\varTheta^0_k(x')\leq\varTheta^0_k(x_k)=1$.
    Then we have
    \begin{equation}\label{eq:errorAlignedOverGrid:3:case2:2}
        \begin{aligned}
        &\tilde{\varTheta}^{0, -}_k(x)=\min_{x'\in(x+d_{k-1}\BB)\cap V_{k-1}}\varTheta^{m_{k-1}}_{k-1}(x')=\tilde{\varTheta}^{m_{k-1}}_{k-1}(x).
        \end{aligned}
    \end{equation}
    Combining \eqref{eq:errorAlignedOverGrid:3:case2:1} and \eqref{eq:errorAlignedOverGrid:3:case2:2} renders at 
    \begin{align}\label{eq:errorAlignedOverGrid:4}
        \|\tilde{\varTheta}_k^{0, -}-\tilde{\varTheta}_k^{*,-}\|_{\XX\setminus\ZZ}\leq\|\tilde{\varTheta}^{m_{k-1}}_{k-1}-\tilde{\varTheta}^*_{k-1}\|_\XX+c_k,
    \end{align}
    which is identical to \eqref{eq:errorAlignedOverGrid:3:case1:3}.
    This verifies that the two cases on $x_k$ render at the same result.

    Notice that \eqref{eq:errorAlignedOverGrid:4} can be fruther relaxed as 
    \begin{equation}\label{eq:errorAlignedOverGrid:5}
        \begin{aligned}
            \|\tilde{\varTheta}^{0, -}_k-\tilde{\varTheta}^*_{k-1}\|_{\XX}=&\|\tilde{\varTheta}^{m_{k-1}}_{k-1}-\tilde{\varTheta}^*_{k-1}\|_\XX\\
            \leq&\|\varTheta^{m_{k-1}}_{k-1}-\varTheta^*_{k-1}\|_{V_{k-1}},
        \end{aligned}
    \end{equation}
    where the inequality follows from \eqref{eq:perturbed2exact} in Lemma~\ref{lemma:perturbedNormInequality}.
    Combining \eqref{eq:errorAlignedOverGrid:1} to \eqref{eq:errorAlignedOverGrid:5} completes the proof.
\end{proof}
The following corollary characterizes the upper bound of the approximation errors on $V_k$ using the error on $G_{k-1}$.
\begin{corollary}\label{corollary:errorAlignedOverGrid}
    The following inequality holds for any $\ell\in\{1, \dots, m_k\}$ with probability one:
    \begin{align*}
        \|\varTheta^\ell_k-\varTheta^*_k\|_{V_k}\leq\|\varTheta^{m_{k-1}}_{k-1}-\varTheta^*_{k-1}\|_{V_{k-1}}+c_k,
    \end{align*}
    where $c_k\triangleq\|\tilde{\varTheta}^*_{k-1}-\tilde{\varTheta}^{*,-}_k\|_{\XX\setminus\ZZ}$ and $\ZZ$ is sufficiently small.
\end{corollary}
\begin{proof}
    Fix $x\in V_k$.
    When $x\in S_k$, it follows from Lemma~\ref{lemma:errorAlignedOverGrid} that $|\varTheta^\ell_k(x)-\varTheta^*_k(x)|\leq\beta_k\|\varTheta^{\ell-1}_k-\varTheta^*_k\|_{V_k}+\beta_kc^k_{k-1}\leq\|\varTheta^{\ell-1}_k-\varTheta^*_k\|_{V_k}+c_k$.
    When $x\in V_k\setminus S_k$, the worst case is that no updates are imposed on $x$; in such case, we have $\varTheta_k^\ell(x)=\varTheta_k^0(x)$.
    Therefore, $|\varTheta^\ell_k(x)-\varTheta^*_k(x)|=|\varTheta^0_k(x)-\varTheta^*_k(x)|\leq\|\varTheta^0_k-\varTheta^*_k\|_{V_k}+c_k$.
    This completes the proof.
\end{proof}

The following corollary characterizes the contraction property over $P$ consecutive graphs.
\begin{corollary}\label{corollary:errorAlignedTowardsPeriodStart}
    Consider $k\in\{\bar{k}+1, \cdots, \bar{k}+P\}$ for $\bar{k}\in\natzero$ and $x\in S_k$. The following holds with probability one:
    \begin{align*}
        |\varTheta^{m_k}_k(x)-\varTheta^*_k(x)|\leq\beta_k^{m_k}\|\varTheta^{m_{\bar{k}}}_{\bar{k}}-\varTheta^*_{\bar{k}}\|_{V_{\bar{k}}}+\beta_k^{m_k}\sum_{\kappa=\bar{k}}^{k-1}c_\kappa,
    \end{align*}
    where $c_\kappa\triangleq\|\tilde{\varTheta}^*_{\kappa-1}-\tilde{\varTheta}^{*,-}_{\kappa}\|_{\XX\setminus\ZZ}$ and $\ZZ$ is sufficiently small.
\end{corollary}
\begin{proof}
    Fix $x\in S_k$.
    Applying Lemma~\ref{lemma:errorAlignedOverGrid} once and iteratively applying Corollary~\ref{corollary:errorAlignedOverGrid} to the left side renders at
    \begin{align*}
        &|\varTheta^{m_k}_k(x)-\varTheta^*_k(x)|\\
        \leq&\beta_k^{m_k}\|\varTheta^{m_{k-1}}_{k-1}-\varTheta^*_{k-1}\|_{V_{k-1}}+\beta_k^{m_k}c_k\\
        \leq&\beta_k^{m_k}\|\varTheta^{m_{k-2}}_{k-2}-\varTheta^*_{k-2}\|_{V_{k-2}}+\beta_k^{m_k}c_{k-1}+\beta_k^{m_k}c_k\\
        \leq&\cdots\leq\beta_k^{m_k}\|\varTheta^{m_{\bar{k}}}_{\bar{k}}-\varTheta^*_{\bar{k}}\|_{V_{\bar{k}}}+\beta_k^{m_k}\sum_{\kappa=\bar{k}}^{k-1}c_{\kappa+1}.
    \end{align*}
    This completes the proof.
\end{proof}

\subsection{Proof of Theorem~\ref{theorem:contractiveOverPeriods}}
In this subsection, we present the proof of Theorem~\ref{theorem:contractiveOverPeriods}.
\begin{proof}[Proof of Theorem~\ref{theorem:contractiveOverPeriods}]
    Let $\bar{k}\in\natural$ and consider an interval $\{\bar{k}, \bar{k}+1, \dots, \bar{k}+P\}$.
    We start with $\|\min_{x'\in(x+d_{\bar{k}+P}\BB)\cap V_{\bar{k}+P}}\varTheta^{m_{\bar{k}+P}}_{\bar{k}+P}-\varTheta^*\|_{\XX\setminus\ZZ}$.
    It follows from the triangular inequality of the maximum norm that 
    \begin{equation}\label{eq:contractiveOverPeriods:01}
        \begin{aligned}
            &\|\min_{x'\in(x+d_{\bar{k}+P}\BB)\cap V_{\bar{k}+P}}\varTheta^{m_{\bar{k}+P}}_{\bar{k}+P}-\varTheta^*\|_{\XX\setminus\ZZ}\\
            =&\|\tilde{\varTheta}^{m_{\bar{k}+P}}_{\bar{k}+P}-\varTheta^*\|_{\XX\setminus\ZZ}\leq\|\tilde{\varTheta}^{m_{\bar{k}+P}}_{\bar{k}+P}-\tilde{\varTheta}^*_{\bar{k}+P}\|_{\XX\setminus\ZZ}+\bB_{\bar{k}+P}\\
            \leq&\aA_{\bar{k}+P}+\bB_{\bar{k}+P},
        \end{aligned}
    \end{equation}
    where $\aA_{\bar{k}+P}\triangleq\|\varTheta^{m_{\bar{k}+P}}_{\bar{k}+P}-\varTheta^*_{\bar{k}+P}\|_{V_{\bar{k}+P}}$ and $\bB_{\bar{k}+P}\triangleq\|\tilde{\varTheta}^*_{\bar{k}+P}-\varTheta^*\|_{\XX\setminus\ZZ}$ and the last inequality follows from \eqref{eq:perturbed2exact} in Lemma~\ref{lemma:perturbedNormInequality}.

    We fix $x\in V_{\bar{k}+P}$ and let $\hat{k}\triangleq\argmax_{\bar{k}\leq k\leq\bar{k}+P}\hat{\FF}_k(x)$ be the last time when value at $x$ is updated.
    Notice that $\hat{k}$ is a function of $x$, but for the conciseness of the proof, we omit the depnedency on $x$ in notation.
    It follows from Lemma~\ref{lemma:iterationIntervalInclusion} that $\hat{k}\geq\bar{k}$ and $x\in S_{\hat{k}}$.
    Then the following holds:
    \begin{equation}\label{eq:contractiveOverPeriods:1}
        \begin{aligned}
            &|\varTheta^{m_{\bar{k}+P}}_{\bar{k}+P}(x)-\varTheta^*_{\bar{k}+P}(x)|=|\varTheta^{m_{\hat{k}}}_{\hat{k}}(x)-\varTheta^*_{\bar{k}+P}(x)|\\
            \leq&|\varTheta^{m_{\hat{k}}}_{\hat{k}}(x)-\varTheta^*_{\hat{k}}(x)|+|\varTheta^*_{\hat{k}}(x)-\varTheta^*_{\bar{k}+P}(x)|,
        \end{aligned}
    \end{equation}
    where the inequality is a direct result of the triangular inequality of the maximum norm.

    We first analyze the first term in \eqref{eq:contractiveOverPeriods:1}.
    Since $x\in S_{\hat{k}}$, it follows from Lemma~\ref{lemma:errorAlignedOverGrid} that 
    \begin{align*}
        |\varTheta^{m_{\hat{k}}}_{\hat{k}}(x)-\varTheta^*_{\hat{k}}(x)|\leq\beta_{\hat{k}}^{m_{\hat{k}}}\|\varTheta^{m_{\hat{k}-1}}_{\hat{k}-1}-\varTheta^*_{\hat{k}-1}\|_{V_{\hat{k}-1}}+\beta_{\hat{k}}^{m_{\hat{k}}}c_{\hat{k}}.
    \end{align*}
    Then we apply Corollary~\ref{corollary:errorAlignedTowardsPeriodStart} and it arrives at
    \begin{equation}\label{eq:contractiveOverPeriods:3}
        \begin{aligned}
            |\varTheta^{m_{\hat{k}}}_{\hat{k}}(x)-\varTheta^*_{\hat{k}}(x)|\leq&\beta_{\hat{k}}^{m_{\hat{k}}}\|\varTheta^{m_{\bar{k}-1}}_{\bar{k}-1}-\varTheta^*_{\bar{k}-1}\|_{V_{\bar{k}-1}}+\cC_{\bar{k}-1}(x)\\
            =&\beta_{\hat{k}}^{m_{\hat{k}}}\aA_{\bar{k}-1}+\cC_{\bar{k}-1}(x),
        \end{aligned}
    \end{equation}
    where $\cC_{\bar{k}}(x)\triangleq\beta_{\hat{k}}^{m_{\hat{k}}}\sum_{\kappa=\bar{k}}^{\hat{k}}c_{\kappa}$ is a function of $x$, as $\hat{k}$ inside the definition is a function of $x$.

    We then analyze the second term. 
    It follows from the triangular inequality that the second term renders at $\forall x\in V_k$, 
    \begin{align}\label{claim:contractiveOverPeriods:2}
        |\varTheta^*_{\hat{k}+1}(x)-\varTheta^*_{\bar{k}+P}(x)|\leq\bar{\CC}_{\hat{k}+1}(x)+\bar{\CC}_{\bar{k}+P}(x),
    \end{align}
    where $\bar{\CC}_k(x)\triangleq|\varTheta^*_k(x)-\varTheta^*(x)|, \forall x\in V_k$ and $\bar{\CC}_k(x)=0$ otherwise.
    Combining \eqref{eq:contractiveOverPeriods:3} and \eqref{claim:contractiveOverPeriods:2}, we have 
    \begin{equation}\label{eq:contractiveOverPeriods:4}
        \begin{aligned}
            &|\varTheta^{m_{\bar{k}+P}}_{\bar{k}+P}(x)-\varTheta^*_{\bar{k}+P}(x)|\\
            \leq&\beta_{\hat{k}}^{m_{\hat{k}}}\aA_{\bar{k}-1}+\cC_{\bar{k}-1}(x)+\bar{\CC}_{\hat{k}+1}(x)+\bar{\CC}_{\bar{k}+P}(x).
        \end{aligned}    
    \end{equation}

    Since \eqref{eq:contractiveOverPeriods:4} holds for every $x\in V_{\bar{k}+P}$, taking the maximum over $V_{\bar{k}+P}$ renders at
    \begin{equation}\label{eq:contractiveOverPeriods:5}
        \begin{aligned}
            &\aA_{\bar{k}+P}=\|\varTheta^{m_{\bar{k}+P}}_{\hat{k}+P}(x)-\varTheta^*_{\bar{k}+P}(x)\|_{V_{\bar{k}+P}}\\
            \leq&\bar{\beta}_{\bar{k}}\aA_{\bar{k}-1}+\max_{x\in V_{\bar{k}+P}}\Big[\cC_{\bar{k}-1}(x)+\bar{\CC}_{\hat{k}+1}(x)+\bar{\CC}_{\bar{k}+P}(x)\Big],
        \end{aligned}    
    \end{equation}
    where $\displaystyle\bar{\beta}_{\bar{k}+P}\triangleq\max_{\bar{k}\leq k\leq\bar{k}+P}\beta_k^{m_k}$.
    Now we would like to show the maximum term in \eqref{eq:contractiveOverPeriods:5} converges to zero by Lemma~\ref{lemma:maxConvergence}.
    \begin{claim}\label{claim:contractiveOverPeriods:convergentMaximumTerms}
        The following holds true:
        \begin{equation*}
            \begin{aligned}
                \lim_{\bar{k}\to\infty}\max_{x\in V_{\bar{k}+P}}\Big[\cC_{\bar{k}-1}(x)+\bar{\CC}_{\hat{k}+1}(x)+\bar{\CC}_{\bar{k}+P}(x)\Big]=0.
            \end{aligned}
        \end{equation*}
    \end{claim}
    \begin{proof}
        We first show $\cC_{\bar{k}-1}(x)$ converges to zero for every $x\in V_{\bar{k}+P}$.    
        For notational simplicity, we use $\cC_{\bar{k}}(x)$ instead of $\cC_{\bar{k}-1}(x)$ in this proof.
        Fix $x\in V_{\bar{k}+P}$.
        Since $\beta_{k}\leq 1$, we have $\cC_{\bar{k}}(x)\leq\sum_{\kappa=\bar{k}}^{\bar{k}+P}c_{\kappa}$, a finite sum of $c_{\kappa}$.
        It follows from Corollary~\ref{corollary:uniformConvergence} that 
        \begin{align*}
            &\lim_{\kappa\to+\infty}c_{\kappa}=\lim_{\kappa\to+\infty}\|\tilde{\varTheta}^*_{\kappa-1}-\tilde{\varTheta}^{*,-}_{\kappa}\|_{\XX\setminus\ZZ}\\
            \leq&\lim_{\kappa\to+\infty}\|\tilde{\varTheta}^*_{\kappa-1}-\varTheta^*\|_{\XX\setminus\ZZ}+\lim_{\kappa\to+\infty}\|\varTheta^*-\tilde{\varTheta}^{*,-}_{\kappa}\|_{\XX\setminus\ZZ}=0.
        \end{align*}
        Then we have 
        \begin{align*}
            \lim_{\bar{k}\to+\infty}\bar{\cC}_{\bar{k}}(x)\leq\lim_{\bar{k}\to+\infty}\sum_{\kappa=\bar{k}}^{\bar{k}+P}c_{\kappa}=\sum_{i=0}^{P}\lim_{\bar{k}\to+\infty}c_{\bar{k}+i}=0.
        \end{align*}
        Since $\bar{\cC}_{\bar{k}}(x)\geq0$, this completes the proof of $\lim_{\bar{k}\to+\infty}\cC_{\bar{k}}(x)=0$.
        
        Then we show the rest two terms also converge to zero.
        It follows from Corollary~\ref{corollary:pointwiseConvergenceOnGrid} that both $\bar{\CC}_{\hat{k}+1}(x)$ and $\bar{\CC}_{\bar{k}+P}(x)$ converges to zero for every $x\in V_{\bar{k}+P}$ as $\bar{k}\to+\infty$.
        Then it follows from Lemma~\ref{lemma:maxConvergence} that $\lim_{\bar{k}\to\infty}\max_{x\in V_{\bar{k}+P}}\Big[\cC_{\bar{k}-1}(x)+\bar{\CC}_{\hat{k}+1}(x)+\bar{\CC}_{\bar{k}+P}(x)\Big]=0.$      
        This completes the proof.
    \end{proof}
    For brevity, we rewrite \eqref{eq:contractiveOverPeriods:5} as
    \begin{equation}\label{eq:contractiveOverPeriods:6}
        \begin{aligned}
            \aA_{\breve{k}(\tau)}
            \leq&\bar{\beta}_{\breve{k}(\tau)}\aA_{\breve{k}(\tau-1)}+\dD_{\breve{k}(\tau)},
        \end{aligned}
    \end{equation}
    where $\breve{k}(\tau)\triangleq\bar{k}+\tau(P+1)$ and $\displaystyle\dD_{\bar{k}+P}\triangleq\max_{x\in V_{\bar{k}+P}}\Big[\cC_{\bar{k}-1}(x)+\bar{\CC}_{\hat{k}+1}(x)+\bar{\CC}_{\bar{k}+P}(x)\Big]$.
    Notice that $\lim_{k\to+\infty}\dD_k=0$ as a result of Claim~\ref{claim:contractiveOverPeriods:convergentMaximumTerms}.
    By Assumption~\ref{asmp:discountWindow}, we continue to relax $\aA_{\breve{k}(\tau-1)}$ over $T$ intervals, each of which consists of $P+1$ graphs, and it holds for any $\zeta\in\natzero$ that
    \begin{equation}
        \begin{aligned}
            &\aA_{\breve{k}((\zeta+1)T)}\\
            \leq&\prod_{\tau=\zeta T+1}^{(\zeta+1)T}\bar{\beta}_{\breve{k}(\tau)}\aA_{\breve{k}(\zeta T)}+\sum_{\tau=\zeta T+1}^{(\zeta+1)T}\dD_{\breve{k}(\tau)}\prod_{\tau'=\tau+1}^{(\zeta+1)T}\bar{\beta}_{\breve{k}(\tau')}\\
            \leq&\bar{\beta}\aA_{\breve{k}(\zeta T)}+\breve{\dD}_{\zeta T},
        \end{aligned}
    \end{equation}   
    where $\breve{\dD}_{\zeta T}\triangleq\sum_{\tau=\zeta T+1}^{(\zeta+1)T}\dD_{\breve{k}(\tau)}\prod_{\tau'=\tau+1}^{(\zeta+1)T}\bar{\beta}_{\breve{k}(\tau')}$ is a finite sum of $\dD_k$.
    Since $\displaystyle\lim_{\tau\to+\infty}\dD_{\breve{k}(\tau)}=\lim_{k\to+\infty}\dD_k=0$, $\breve{\dD}_{\zeta T}$ converges to $0$ as $\zeta\to+\infty$.
    It follows from Lemma~VII.5 in \cite{zhao2020pareto} that $\displaystyle\lim_{k\to+\infty}\aA_k=\lim_{\zeta\to+\infty}\aA_{\breve{k}(\zeta T)}=0$.
    It follows from Corollary~\ref{corollary:uniformConvergence} that $\lim_{\bar{k}\to+\infty}\bB_{\bar{k}}=0$.
    Together with \eqref{eq:contractiveOverPeriods:01}, we have 
    \begin{align*}
        &\lim_{k\to+\infty}\|\min_{x'\in(x+d_k\BB)\cap V_k}\varTheta^{m_k}_k-\varTheta^*\|_{\XX\setminus\ZZ}\\
        =&\lim_{\bar{k}\to+\infty}\|\min_{x'\in(x+d_{\bar{k}+P}\BB)\cap V_{\bar{k}+P}}\varTheta^{m_{\bar{k}+P}}_{\bar{k}+P}-\varTheta^*\|_{\XX\setminus\ZZ}\\
        \leq&\lim_{\bar{k}\to+\infty}\aA_{\bar{k}+P}+\lim_{\bar{k}\to+\infty}\bB_{\bar{k}+P}=0.
    \end{align*}
    This completes the proof.
\end{proof}

\section{Numerical Results and Discussion}\label{sec:Results}
In this section, we first present a computation-saving adaptation of the proposed iPolicy algorithm for a class of simpler but widely used dynamic systems.
Then we present results of numerical experiments for three differential systems - point mass, simple car and Dubins car. 
Using the point mass, we show the correctness of the theoretical results. 
We show several simulations for unicycle robots in obstacle-free as well as cluttered environment to perform complex maneuvers like automatic parking. 
The experiments presented in this section are done to investigate the following characteristics of the proposed iPolicy algorithm.
\begin{enumerate}
    \item Can iPolicy recover the optimal value function for systems with differential constraints?
    \item Can the state-based feedback controllers handle complex motion maneuvers in the presence of obstacles?
    \item Can we demonstrate the anytime and incremental nature of the algorithm? Is the algorithm able to monotonically improve the plan towards the goal?
\end{enumerate}
In the rest of this section, we will try to answer the above questions with systems of increasing complexity.

\begin{figure*}[h] 
	\centering
	\subfloat[The estimated value function with $761$ samples after $21.8$ seconds.]{
		\label{fig:point-mass:1}
		\includegraphics[width=0.48\textwidth]{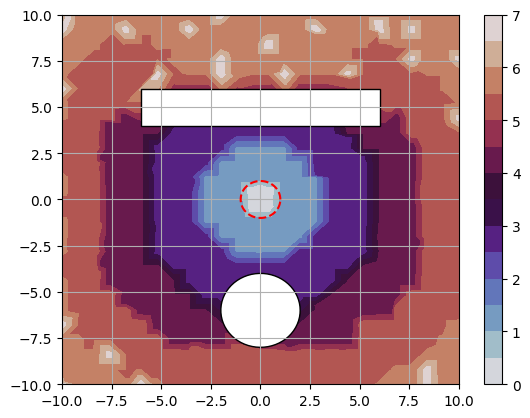}} 
	\subfloat[The estimated value function with $1,934$ samples after $205.7$ seconds.]{
		\label{fig:point-mass:2}
		\includegraphics[width=0.48\textwidth]{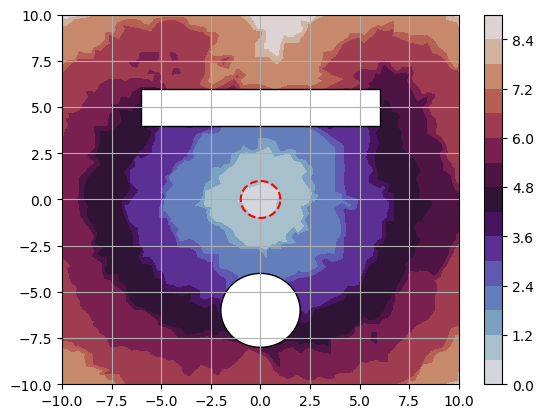}}\\
    \subfloat[The estimated value function with $16,979$ samples after $20062.7$ seconds.]{
		\label{fig:point-mass:3}
		\includegraphics[width=0.48\textwidth]{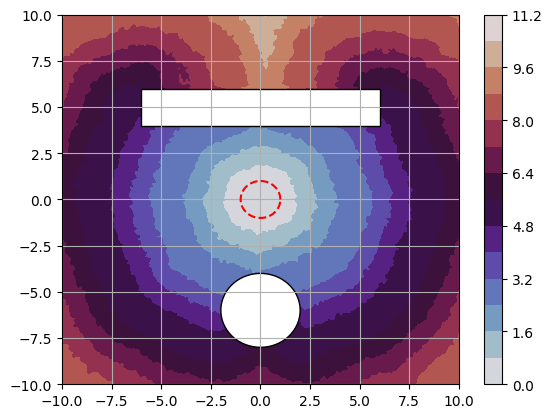}}
		\subfloat[The root of mean squared errors over total number of samples on $5$ runs.]{
		\label{fig:point-mass:error}
		\includegraphics[width=0.48\textwidth]{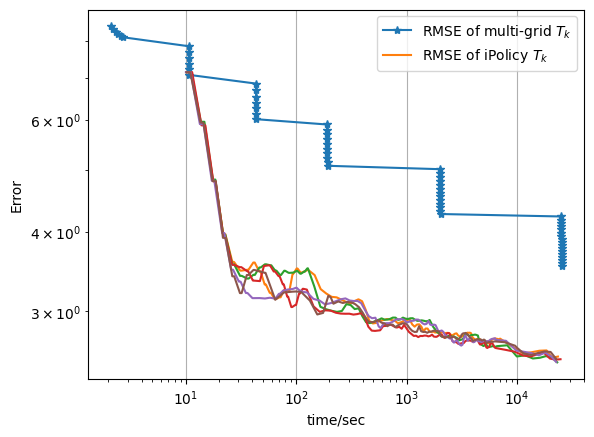}}
	\caption{The estimated value function obtained for a point mass in the presence of obstacles and its convergence with the computational time. Goal region for the point-mass is centered at $(0,0)$ with red dash line. The colorbar in the plots represent the approximate time to the goal region. Figure~\ref{fig:point-mass:error} shows the errors over $5$ independent runs. As seen in Figure~\ref{fig:point-mass:error}, iPolicy achieves faster convergence compared to the multigrid method.}
	\label{fig:point-mass}
\end{figure*}

\subsection{Computation-saving query}
In this subsection, we introduce a class of simpler but widely used dynamical systems in the robotics community and an adaptation for iPolicy that significantly reduces computations in practice but maintains the optimality in Theorem~\ref{theorem:mainResult}.
Consider a class of systems that, in addition to the satisfaction of Assumption~\ref{asmp:system}, also satisfy the following assumption:
\begin{assumption}\label{asmp:stoppable}
    System~\eqref{eqn:systemeqn} satisfies that for any $x\in\XX, \exists u\in\UU$ s.t. $f(x, u)=0$.
\end{assumption}
In other words, Assumption~\ref{asmp:stoppable} indicates the robot can instantly stop at any state.
The following lemma shows the one-hop neighbors $F_k(x)$ is monotonically decreasing in $k$.
\begin{lemma}\label{lemma:simpler}
    Suppose Assumptions~\ref{asmp:system}, \ref{asmp:resolutions} and \ref{asmp:stoppable} hold.
    Then it holds that $F_{k'}(x)\subseteq F_k(x)\cup(V_{k'}\setminus V_k), \forall x\in\XX_{\mathrm{free}}$ and $k'\geq k$.
\end{lemma}
\begin{proof}
    It follows from Assumption~\ref{asmp:resolutions} that $d_k, \epsilon_k$ and $\rho(d_k, \epsilon_k)$ are monotonically decreasing in $k$.
    Recall the definition of $F_k(x)$ in \eqref{eq:onehops} follows $F_k(x)=(x+\epsilon_k\bigcup_{u\in\UU}f(x, u)+\rho(d_k, \epsilon_k)\BB)\cap V_k\cap(\XX_{\mathrm{free}}+d_k\BB)$.
    Fix $x\in\XX_{\mathrm{free}}$ and $x'\in F_{k'}(x)$.
    We proceed to prove $x'\in F_k(x)\cup(V_{k'}\setminus V_k)$ by showing $x'$ is included in each part of $F_k(x)$
    
    We first show $x'\in x+\epsilon_k\bigcup_{u\in\UU}f(x, u)+\rho(d_k, \epsilon_k)\BB$.
    Since $x'\in x+\epsilon_{k'}\bigcup_{u\in\UU}f(x, u)+\rho(d_{k'}, \epsilon_{k'})\BB$, there exist $u'\in\UU, b\in\BB$ and $\alpha\in[0, \rho(d_{k'}, \epsilon_{k'})]$ s.t. $x'=x+\epsilon_{k'}f(x, u')+\alpha b$.
    It follows from Assumptions~\ref{asmp:convexity} and \ref{asmp:stoppable} that for any $u\in\UU$ and $\zeta\in[0,1]$, $\zeta f(x, u)\in\bigcup_{u\in\UU}f(x, u)$.
    By the decreasing monotonicity of $\epsilon_k$ in $k$, this implies $\epsilon_{k'}f(x, u')=\frac{\epsilon_{k'}}{\epsilon_k}\epsilon_kf(x,u')\in\epsilon_k\bigcup_{u\in\UU}f(x,u)$.
    It again follows from the decreasing monotonicity of $\rho(d_k, \epsilon_k)$ that $\alpha b\in \rho(d_k, \epsilon_k)\BB$.
    In summary, we have $x'\in x+\epsilon_k\bigcup_{u\in\UU}f(x, u)+\rho(d_k, \epsilon_k)\BB$.

    Then we show $x'\in V_k\cup(V_{k'}\setminus V_k)$. 
    Since $V_k\cup(V_{k'}\setminus V_k)=V_{k'}$, this trivially holds.

    Finally, we show $x'\in\XX_{\mathrm{free}}+d_k\BB$.
    Since $x'\in\XX_{\mathrm{free}}+\rho(d_{k'}, \epsilon_{k'})\BB$, there exist $x_f\in\XX_{\mathrm{free}}, b\in\BB$ and $\alpha\in[0, d_{k'}]$ that $x'=x_f+\alpha b$.
    It follows from the decreasing monotonicity of $d_k$ that $\alpha b\in d_k\BB$.
    The proof of $x'\in\XX_{\mathrm{free}}+d_k\BB$ is completed.

    The above arguments show that $x'\in F(x)\cup(V_{k'}\setminus V_k)$. 
    Since this holds for every $x'\in F_{k'}(x)$ and $x\in\XX_{\mathrm{free}}$, the proof is completed.
\end{proof}

\begin{algorithm}[t] \small
    \caption{Computation-saving query of $F$}\label{alg:queryOneHops}
        \textbf{Input:} Graph index $k$, newly added sample $x_{\mathrm{new}}$\;
        Compute $F(x_{\mathrm{new}})$ as \eqref{eq:onehops}\;\label{alg:queryOneHops:xnew}
        \For{$x\in V_k\setminus\{x_{\mathrm{new}}\}$}{\label{alg:queryOneHops:xnew_update:begin}
            \If{$x_{\mathrm{new}}\in x+\epsilon_k\bigcup_{u\in\UU}f(x,u)+\rho(d_k, \epsilon_k)\BB$}{
                $F(x)\leftarrow F(x)\cup\{x_{\mathrm{new}}\}$\;
            }
        }\label{alg:queryOneHops:xnew_update:end}
        \For{$x\in V_k\setminus\{x_{\mathrm{new}}\}$}{\label{alg:queryOneHops:query:begin}
            \For{$x'\in F(x)$}{
                \If{$x'\notin(x+\epsilon_k\bigcup_{u\in\UU}f(x, u)+\rho(d_k, \epsilon_k)\BB)$}{
                    $F(x)\leftarrow F(x)\setminus\{x'\}$\;
                }
                \If{$x'\notin(\XX_{\mathrm{free}}+d_k\BB)$}{
                    $F(x)\leftarrow F(x)\setminus\{x'\}$\;
                }
            }
        }\label{alg:queryOneHops:query:end}
        \Return $F$
\end{algorithm}

Based on Lemma~\ref{lemma:simpler}, we can construct $F_k(x)$ in line~\ref{alg:backprop:onehops} of Algorithm~\ref{alg:backprop} \texttt{BackProp} in a computationally efficient manner, where the execution of the computationally heaviest part, querying $V_k$ within a range of given a random state $x$, is reduced.
Specifically, we treat $F_k(x)$ as a set of samples independent of graph index $k$; i.e., $F_k(x)=F(x)$.
See Algorithm~\ref{alg:queryOneHops}.
When $x$ is added to $V_k$ as $x_{\mathrm{new}}$, we initialize $F(x)$ as its definition \eqref{eq:onehops} using the resolutions $\epsilon_k$ and $d_k$ at that time, and update the one-hop neighbors of every other sampled state as lines~\ref{alg:queryOneHops:xnew_update:begin} to \ref{alg:queryOneHops:xnew_update:end}; when querying for the one-hop neighbors of $x$ in a later iteration, we update $F(x)$ by removing faraway one-hop neighbors according to new resolutions as lines \ref{alg:queryOneHops:query:begin} to \ref{alg:queryOneHops:query:end}.
In contrast to the vanilla query command as line~\ref{alg:backprop:onehops} in \texttt{BackProp}, where searching the whole $V_k$ is executed every time when $F(x)$ is queried, the new adaptation skips this step and saves computations.
This adaptation is leveraged throughout the results in Section~\ref{sec:Results}.

\subsection{Point mass}\label{sec:Results:pointmass}
In this subsection, we try to show the correctness of the algorithm and the rate of convergence through a number of numerical simulations for a point mass.
The point mass dynamic follows a single integrator $\begin{bmatrix}\dot{x}\\\dot{y}\end{bmatrix}=\begin{bmatrix}u_1\\u_2\end{bmatrix}$, where the velocity on each coordinate can be directly manipulated.
The control set is a unit disc; i.e., the maximum velocity is restricted to $1$.
The environment is a square of size $20\times20$ populated with a rectangular obstacle and a circular obstacle, and the robot desires to rest at a central circle.
We used $\epsilon_k=(5d_k)^{2/3}$ and $\rho(\epsilon_k, d_k)=2d_k$ for computational simplicity.
The experiment used $P=50$ and $m_k=500$; i.e., every vertex was updated at-most after $50$ iterations and the recursion depth is a constant of $500$.
Notice that a constant recursion depth does not necessarily satisfy Assumption~\ref{asmp:discountWindow}, but the results still show its applicability in practice.

Figures~\ref{fig:point-mass:1} to \ref{fig:point-mass:3} qualitatively show the estimated value function over the whole environment after the computations of $21.8$ seconds, $205.7$ seconds and $20062.7$ seconds, respectively, where the goal region is marked by the red dash circle. 
It is observed that the estimate of minimal traveling time to the goal region is monotonically improving.

Quantitative characterization of the approximation error is available, as the minimal traveling time function $\dsT^*$ can be analytically computed for a point mass in the presence of regularly shaped obstacles.
The approximation error is computed between the estimated value function $\dsT_k$ and the ground truth $\dsT^*$ and is characterized by the root of mean squared error (RMSE), which follows as $RMSE\triangleq\sqrt{\frac{1}{|V_k|}\sum_{x\in V_k}|\dsT_k(x)-\dsT^*(x)|^2}$.
Notice that the estimated value function $\dsT_k$ can be positive infinity when value iterations are not sufficiently executed and we exclude them in the comparison.
We compare iPolicy with the multi-grid method, the single-robot version of \cite{zhao2020pareto}.
We use $4$ grids with spatial resolutions $\{0.8,0.4,0.2,0.1\}$ and compute the temporal resolutions in the same way as \cite{zhao2020pareto} for the multi-grid method.
All other parameters share the same values with those of iPolicy.
Figure~\ref{fig:point-mass:error} quantitatively visualizes RMSEs over computational times for $5$ independently executed experiments.
It is shown that the approximation error of iPolicy monotonically and continuously decreases in general as more computational time is consumed, and iPolicy obtains lower approximation errors compared to the multi-grid method within the allowed computational resources.

\subsection{Simple car}
In this subsection, we first show the optimality results with the simple car model.
The simple car model follows the unicycle dynamic, a well-known driftless affine system in the configuration space $\real^2\times \mathds S^1$~\cite{L06}. 
The unicycle dynamics is given below:
\begin{align*}
    \begin{bmatrix}
        \dot{x}\\\dot{y}\\\dot{\theta}
    \end{bmatrix} 
    = \begin{bmatrix}
        \cos\theta & 0  \\
        \sin\theta & 0  \\
        0 & 1
    \end{bmatrix}
    \begin{bmatrix}
        u_1\\u_2
    \end{bmatrix},
\end{align*}
where $\theta$ is the steering angle and is a circle $\mathds S^1$.
For computational simplicity, we redefine $\mathds S^1$ by $\theta\in[-\Theta, \Theta]/\sim$, where $\sim$ indicates $-\Theta$ and $\Theta$ are identical, and by using the Carnot-Carath{\'e}odory distance.
To avoid the bias over different coordinates in computing distances, we let the region of interest be a square of $20\times20$ and let $\Theta=10$.
The control set of the simple car is given by $U=[-1,1]\times [-1,1]$, and the goal set is a sub-Riemannian ball of radius $1$ centered at $(0,0,0)$.
Figure~\ref{fig:simple-car-value} shows the estimated value function after approximately $13$k seconds of computations over $5,000$ samples on the plane, where the orientation of the car is roughly parallel to the orientation at goal ($\theta_{\mathrm{goal}}=0$) with a maximum difference of $10^{\circ}$.
The shape of the level sets around the goal depict that the minimal traveling time is higher in the transverse direction when compared to that along the longitudinal direction parallel to the orientation of the goal state, as the car has a non-zero turn radius and it cannot move in the  transverse direction.

\begin{figure}[h] 
	\centering
    \includegraphics[width=0.5\textwidth]{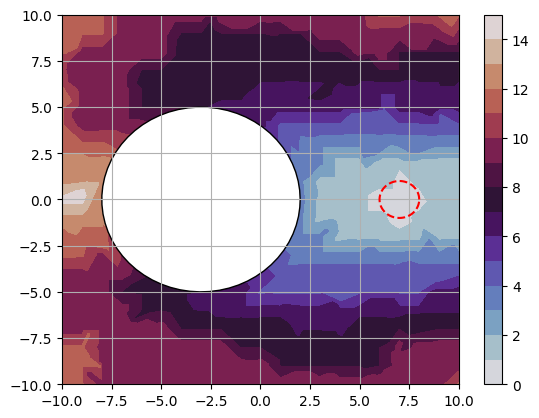}
	\caption{The estimated value function obtained by the proposed incremental algorithm for simple car obtained over $5,000$ samples with orientation $\theta=0$.
    }
	\label{fig:simple-car-value}
\end{figure}

\begin{figure}[h]
	\centering
	\subfloat[The path of head-in parking]{
		\label{fig:simple-car-parking:head-in}
		\includegraphics[width=0.45\textwidth]{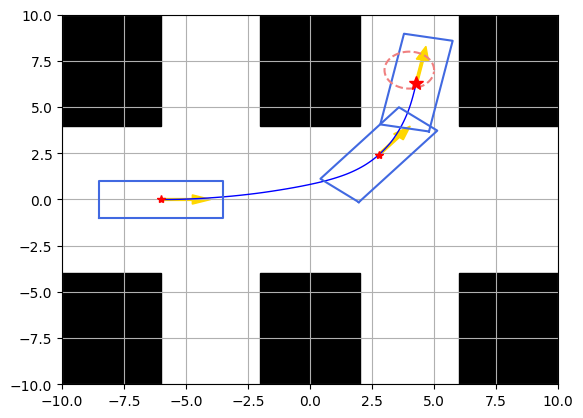}} \\
	\subfloat[The path of parallel parking]{
		\label{fig:simple-car-parking:parallel}
		\includegraphics[width=0.45\textwidth]{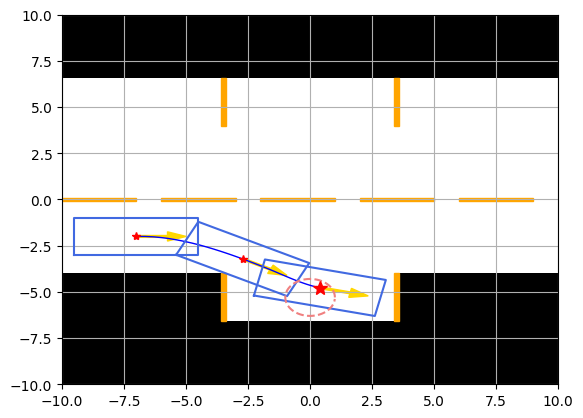}}
	\caption{Trajectory of the simple car accomplishing automated parking in the cluttered environment. Yellow arrows denote the orientation of the car while red stars denote the center. Goal region is marked by a pink dash circle.}
	\label{fig:simple-car-parking}
\end{figure}

In another simulation, we demonstrate the capability of iPolicy of handling automated parking, including head-in parking and parallel parking, of the simple car in the cluttered environment and show anytime property of the algorithm by visualizing the resulting trajectory under the guidance of the computed controller.
See Figure~\ref{fig:simple-car-parking:head-in} and Figure~\ref{fig:simple-car-parking:parallel}.
In the first simulation, a simple car desires to park in a one-way $90^\circ$ head-in-only parking lot, and it starts from $(-6,0,0)$ and needs to safely stop at $(4,7,\pi/2)$.
In the second simulation, we simulate the parallel parking of a simple car on a two-way street, where the car starts from $(-7,-2,0)$ and needs to park at $(4,7,\pi/2)$ without causing any collisions.
In Figure~\ref{fig:simple-car-parking:head-in}, iPolicy uses $290$ samples and runs $53.71$ seconds to compute a safe and feasible controller for head-in parking in Figure~\ref{fig:simple-car-parking:head-in}.
In Figure~\ref{fig:simple-car-parking:parallel}, iPolicy uses $200$ samples and runs $22.23$ seconds to command the simple car to safely park in parallel.
Both cases imply the quick feasibility of iPolicy.
Together with Theorem~\ref{theorem:mainResult} and Figure~\ref{fig:point-mass:error}, where the increasing optimality is verified, the anytime property of iPolicy is demonstrated.
That is, iPolicy can produce a feasible solution given a short computational time, and can continuously optimize its solution if more computational time is given.

\begin{figure}[h] 
	\centering
    \includegraphics[width=0.5\textwidth]{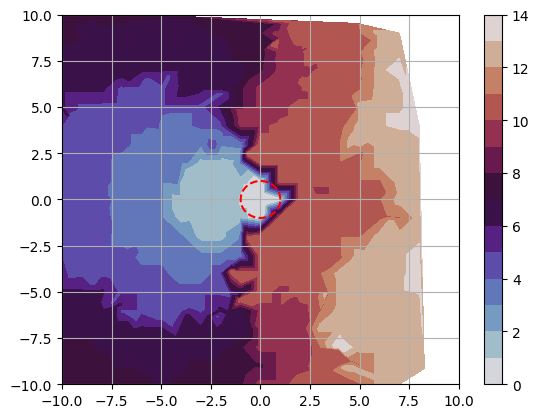}
	\caption{The estimated value function obtained by iPolicy for Dubins car system obtained over $7500$ samples. As the Dubins car can't move backwards, the value functions are discontinuous and it results in more complex reachable sets in the sub-Riemannian manifold.}
	\label{fig:dubins-car}
\end{figure}

\subsection{Dubins car}
In this subsection, we show Algorithm~\ref{algorithm:iFPA} can handle more complex dynamic models.
In particular, we also show the estimate of the value functions obtained using iPolicy for the Dubins car whose dynamics also lies in the $\real^2 \times \mathds{S}^1$ space and is described by the same set of equations as the simple car; however, it can only move forward.
Different from the Dubins car in \cite{LP14} that can instantly stop, the Dubins car can only move forward at a constant speed, making its control set a singleton $U=\{1\}\times [-1,1]$. 
The Dubins car is another example of nonholonomic system but without a reverse gear or a brake \cite{KF13, L06}. 
Moreover, the Dubins car does not satisfy Assumption~\ref{asmp:stoppable}; however, we still apply Algorithm~\ref{alg:queryOneHops} in implementation and later result will show this quick adaptation does not affect the 
The goal set in this case is the sub-Riemannian ball of radius $1$ around the point $(0,0,0)$ and $R=1$ (again the goal is the point where we reach first inside the ball).
In Figure~\ref{fig:dubins-car}, the estimated cost-to-go on a plane are shown where the orientation of the car is parallel to the orientation at goal $\theta_{\rm goal}=0$ with a maximum difference of $10^\circ$. 
The estimate of the value functions show the more complex level sets obtained for the Dubins car as it can not move backwards. 
Unlike the simple car, the level sets are not rotationally symmetric about the origin.
Notice that area on the right-hand side of $x=7.5$ in Figure~\ref{fig:dubins-car} is mostly empty; this indicates iPolicy cannot find a solution in this area, mostly because the Dubins car starting from this area cannot reach the goal region without exiting the region of interest.
\begin{remark}
    Applying iPolicy to larger size environment does not require any additional modifications.
    However, at the early stage of computations, smaller size of samples may not display the desired probabilistic properties and some minor revisions may be needed for quick feasibility. 
    Specifically, we can revise $\epsilon_k$ s.t. the one-hop neighbor of a sample does not exit the region of interest; i.e., $x+f(x, u)\epsilon_k+\rho(\epsilon_k, d_k)\BB\subseteq\XX$ for most $x\in V_k$ and $u\in\UU$.
\end{remark}

\section{Conclusions and Future Work}\label{sec:conclusions}
In this paper we presented an algorithm for feedback motion planning of dynamical systems. 
Using results from set-valued control theory and sampling-based algorithms, we presented the iPolicy algorithm which can be used for feedback motion planning of robotic systems and guarantees asymptotic optimality of the value functions. 
Numerical experiments with point mass, a simple car model and the Dubins car model are presented where the time-to-go costs are recovered using iPolicy.
We discussed an asynchronous update rule of the value functions for computational efficiency of the algorithms and proved that it retains asymptotic optimality guarantees.

It was found that majority of the time of the algorithm is used in local connection and collision checking using the sub-Riemannian metric. 
Computational time is a critical bottleneck for the class of algorithms presented here and needs further research. 
Our future work includes the research on speeding up the calculations.  Using  parallel machine learning algorithms which can assist in connecting the states sampled during state-space construction using sampling could potentially help in relaxing the compute times. 
Learning-based algorithms \cite{chi2023diffusion} are shown to be efficient in addressing high-dimensional problems, which has the potential to help reduce computations of planners and makes it more applicable in practice. 
Furthermore, use of iPolicy for sub-optimal, trajectory-centric control of underactuated systems~\cite{majumdar2013control, kolaric2020local} could also be an interesting topic of future research. Extension to affine or complementarity dynamical systems could be another interesting direction of research~\cite{shirai2023covariance, 9811812, yang2023efficient}.

\bibliographystyle{ieeetr}
\bibliography{references} 

\begin{IEEEbiography}[{\includegraphics[width=1in,height=1.25in,clip,keepaspectratio]{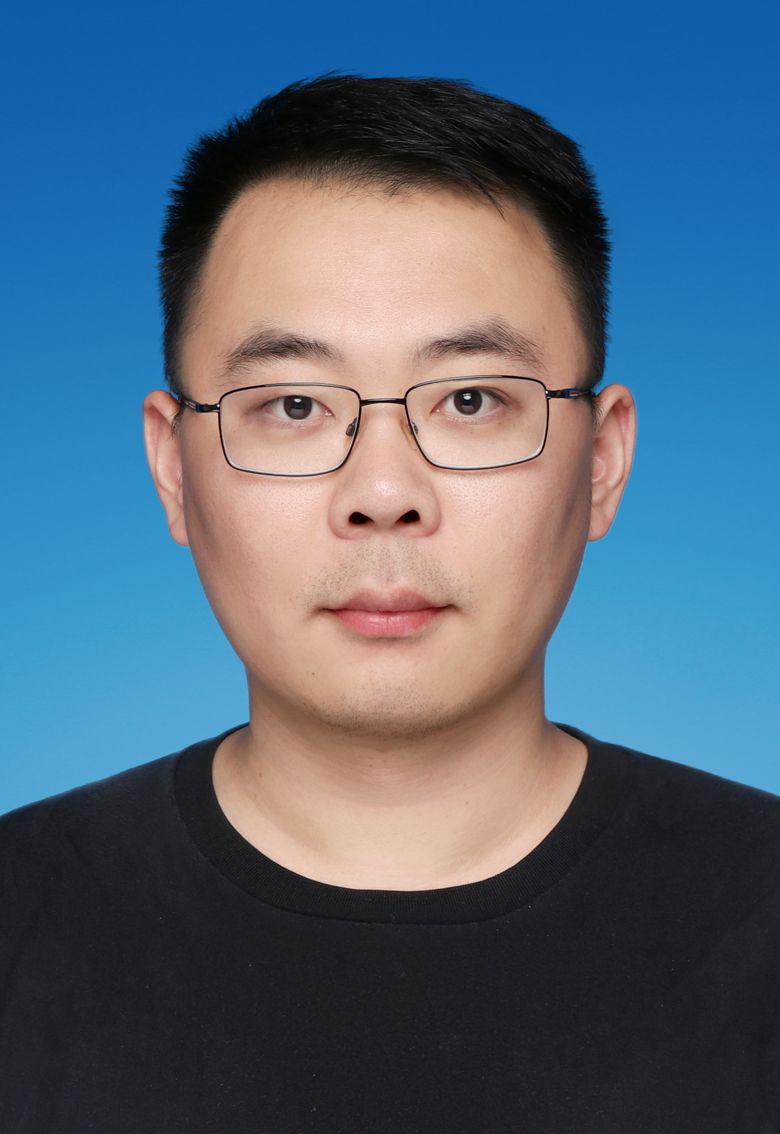}}]%
{Guoxiang Zhao}
is currently an Associate Professor in the School of Future Technology at Shanghai University, Shanghai, China. He received PhD in Electrical Engineering from Pennsylvania State University in May 2022. He received a M.S. degree in Mechanical Engineering from Purdue University, West Lafayette, IN, USA and a B.E. degree from Shanghai Jiao Tong University, Shanghai, China. He was a Software Engineer from 2022 to 2023 with TuSimple, Inc. in San Diego, CA, USA. His research interests are in the areas of motion planning and multi-robot systems. 
\end{IEEEbiography}

\begin{IEEEbiography}[{\includegraphics[width=1in,height=1.25in,clip,keepaspectratio]{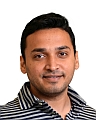}}]%
{Devesh K. Jha}
is currently a Principal Research Scientist at Mitsubishi Electric Research Laboratories (MERL) in Cambridge, MA, USA. He received PhD in Mechanical Engineering from Penn State in Decemeber 2016. He received M.S. degrees in Mechanical Engineering and Mathematics also from Penn State. His research interests are in the areas of Machine Learning, Robotics and Deep Learning. He is a recipient of several best paper awards including the Kalman Best Paper Award 2019 from the American Society of Mechanical Engineers (ASME) Dynamic Systems and Control Division. He is a senior member of IEEE and an associate editor of IEEE Robotics and Automation Letters (RA-L).
\end{IEEEbiography}

\begin{IEEEbiography}[{\includegraphics[width=1in,height=1.25in,clip,keepaspectratio]{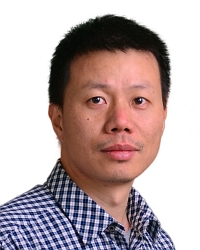}}]%
{Yebin Wang}
(M'10-SM'16) received the B.Eng. degree in Mechatronics Engineering from Zhejiang University, Hangzhou, China, in 1997, M.Eng. degree in Control Theory \& Control Engineering from Tsinghua University, Beijing, China, in 2001, and Ph.D. in Electrical Engineering from the University of Alberta, Edmonton, Canada, in 2008. 

Dr. Wang has been with Mitsubishi Electric Research Laboratories in Cambridge, MA, USA, since 2009, and now is a Senior Principal Research Scientist. From 2001 to 2003 he was a Software Engineer, Project Manager, and Manager of R\&D Dept. in industries, Beijing, China. His research interests include nonlinear control and estimation, optimal control, adaptive systems and their applications including mechatronic systems.

\end{IEEEbiography}

\begin{IEEEbiography}[{\includegraphics[width=1in,height=1.25in,clip,keepaspectratio]{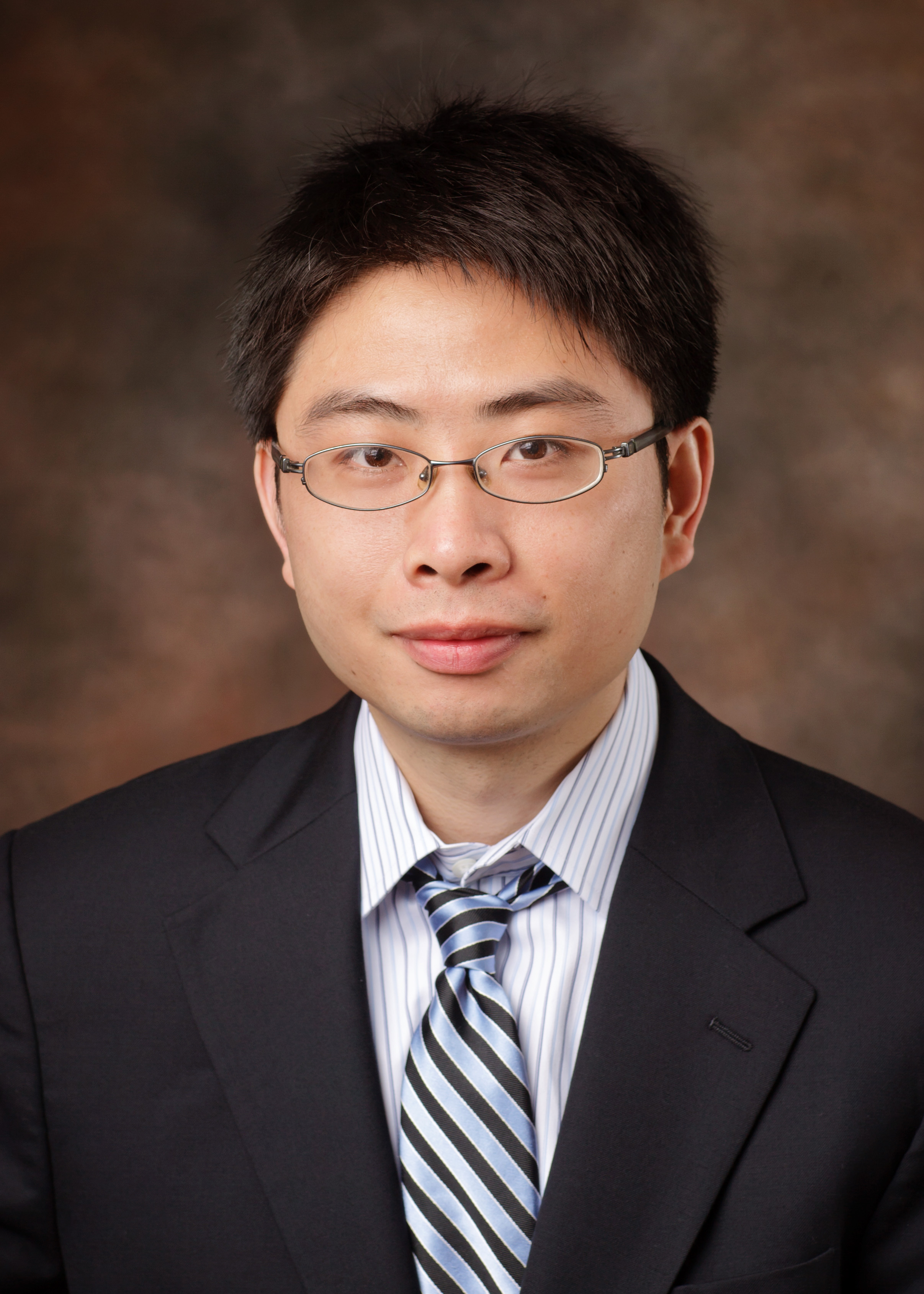}}]%
{Minghui Zhu} is an Associate Professor in the School of Electrical Engineering and Computer Science at the Pennsylvania State University. Prior to joining Penn State in 2013, he was a postdoctoral associate in the Laboratory for Information and Decision Systems at the Massachusetts Institute of Technology. He received Ph.D. in Engineering Science (Mechanical Engineering) from the University of California, San Diego in 2011. His research interests lie in distributed control and decision-making of multi-agent networks with applications in robotic networks, security and the smart grid. He is the co-author of the book "Distributed optimization-based control of multi-agent networks in complex environments" (Springer, 2015). He is a recipient of the award of Outstanding Graduate Student of Mechanical and Aerospace Engineering at UCSD in 2011, the Dorothy Quiggle Career Development Professorship in Engineering at Penn State in 2013, the award of Outstanding Reviewer of Automatica in 2013 and 2014, and the National Science Foundation CAREER award in 2019. He is an associate editor of the IEEE Open Journal of Control Systems, the IET Cyber-systems and Robotics and the Conference Editorial Board of the IEEE Control Systems Society.
\end{IEEEbiography}

\appendix

This section contains preliminary results towards the complete proof of Theorem~\ref{theorem:mainResult}.
Most notations are borrowed from \cite{cardaliaguet1999set} but we make the following changes for the consistency throughout this paper.
Let $K=\XX$ be a closed set and $C=\XX_{\mathrm{goal}}$ be a closed target set of $K$ following the definition in \cite{cardaliaguet1999set}.
We rewrite the minimal time function $\vartheta^K_C$, the estimated value function $T$ and the the spatial resolution $h_p$ in \cite{cardaliaguet1999set} as $\dsT^*$, $\dsT$ and $d_p$ respectively. 
Notice that all superscripts and subscripts of $T$ are kept in $\dsT$.
Denote the domain of $\dsT^*$ by $\DD(\dsT^*)\triangleq\{x\in K|\dsT^*(x)<+\infty\}$.
The following theorem proves a slightly stronger result compared to Corollary 3.7 in \cite{cardaliaguet1999set}, where the perturbation radius can be arbitrarily large.

\begin{theorem}\label{theorem:pointwiseConvergence}
    Given the convergence in the epigraphic sense, i.e., $\Epi(\dsT^*)=\Lim_{p\to+\infty}\Epi(\dsT^\infty_p)$, $\dsT^\infty_p$ converges pointwise to $\dsT^*$, i.e., $\forall x\in\DD(\dsT^*)$, 
    \begin{align*}
        \dsT^*(x)=\lim_{p\to+\infty}\min_{x'\in(x+\eta_p\BB\cap K_{d_p})}\dsT^\infty_p(x'),
    \end{align*}
    where $\lim_{p\to+\infty}\frac{\eta_p}{d_p}\in[1, +\infty)$ and $\eta_p\geq d_p$.
\end{theorem}
\begin{proof}
    Define $\tilde{\dsT}^\infty_p$ s.t. $\Epi(\dsT^\infty_p)+\eta_p\BB=\Epi(\tilde{\dsT}^\infty_p)$.
    \begin{claim}\label{claim:TBound}
        The upper and lower bound of $\dsT^\infty_p(x)$ for $x\in\DD(\dsT^*)$ follows respectively
        \begin{align}\label{eq:TUpperbound}
            \tilde{\dsT}^\infty_p(x)+\eta_p\geq \min_{x'\in(x+\eta_p\BB)\cap K_{d_p}}\dsT^\infty_p(x'),
        \end{align}
        \begin{align}\label{eq:TLowerbound}
            \tilde{\dsT}^\infty_p(x)\leq\min_{x'\in (x+\eta_p\BB)\cap K_{d_p}}\dsT^\infty_p(x').
        \end{align}
    \end{claim}
    \begin{proof}
    For any $x\in\DD(\dsT^*)$ and $(x, \tilde{\dsT}^\infty_p(x))\in\Epi(\tilde{\dsT}^\infty_p)$, there exists $(x', t')\in\Epi(\dsT^\infty_p)$ s.t. $\|(x, \tilde{\dsT}^\infty_p(x))-(x', t')\|\leq\eta_p$.
    This implies $\|x-x'\|\leq\eta_p$ and $|\tilde{\dsT}^\infty_p(x)-t'|\leq \eta_p$.
    Therefore, 
    \begin{align*}
    \tilde{\dsT}^\infty_p(x)+\eta_p\geq t'\geq \dsT^\infty_p(x')\geq\min_{x'\in(x+\eta_p\BB)\cap K_{d_p}}\dsT^\infty_p(x').
    \end{align*}
    This completes the proof of \eqref{eq:TUpperbound}.

    Following the definition of $\tilde{\dsT}^\infty_p$, we have
    \begin{equation}\label{eq:TLowerbound:1}
        \begin{aligned}
            &\tilde{\dsT}^\infty_p(x)=\inf\{t|(x, t)\in\Epi(\dsT^\infty_p)+\eta_p\BB\}\\
            =&\inf\{t|\exists x'\in K_{d_p}, t'\geq \dsT^\infty_p(x')\text{ s.t. }\|(x, t)-(x',t')\|\leq\eta_p\}.
        \end{aligned}
    \end{equation}
    Notice that $\|(x, t)-(x',t')\|\leq\eta_p\Rightarrow\|x-x'\|\leq\eta_p$.
    We can simplify the conditions on $(x, t)$ as 
    \begin{align*}
        \tilde{\dsT}^\infty_p(x)
        =&\inf\{t|\forall x'\in (x+\eta_p\BB)\cap K_{d_p}, t'\geq \dsT^\infty_p(x')\\
        &\quad\text{ s.t. }\|(x, t)-(x',t')\|\leq\eta_p\}.
    \end{align*}
    It again follows from \eqref{eq:TLowerbound:1} that $|t-t'|\leq\eta_p$.
    Since $\tilde{\dsT}^\infty_p(x)$ searches for an infimum of $t$, taking an infimum of $t'$ would render at the same result but remove unattainable candidate minimizers.
    Thus, the following holds by letting $t'=\dsT^\infty_p(x')$:
    \begin{equation}
        \begin{aligned}
            &\tilde{\dsT}^\infty_p(x)=\inf\{t|\forall x'\in (x+\eta_p\BB)\cap K_{d_p}\\
            &\quad\quad\quad\quad\quad\text{ s.t. }\|(x, t)-(x',\dsT^\infty_p(x'))\|\leq\eta_p\}\\
            =&\inf\{t|\forall x'\in (x+d_p\BB)\cap K_{d_p}, (x, t)\in(x', \dsT^\infty_p(x'))+\eta_p\BB\}\\
            \leq&\min_{x'\in (x+\eta_p\BB)\cap K_{d_p}}\dsT^\infty_p(x').
        \end{aligned}
    \end{equation}
    where the last inequality is a direct result of expansion of $d_p\BB$ on $(x', \dsT^\infty_p(x'))$.
    This completes the proof of \eqref{eq:TLowerbound}.
    \end{proof}
    
    We proceed to show that $\dsT^*(x)$ is bounded both from below and above by $\dsT^\infty_p(x)$.
    The following claim shows $\dsT^*(x)$ is lower bounded by $\dsT^\infty_p(x)$.
    \begin{claim}\label{claim:lowerbound}
        It holds that for any $x\in\DD(\dsT^*)$,
        \begin{align*}
            \dsT^*(x)\geq\min_{x'\in(x+\eta_p\BB)\cap K_{d_p}}\dsT^\infty_p(x')-\eta_p.
        \end{align*}
    \end{claim}
    \begin{proof}
        Following the argument in the proof of Corollary~3.7 in \cite{cardaliaguet1999set}, we have $\Epi(\dsT^*)\subseteq\Epi(\tilde{\dsT}^\infty_p)$ and it renders at $\tilde{\dsT}^\infty_p(x)\leq\dsT^*(x)$.
        Then it follows from \eqref{eq:TUpperbound} in Claim~\ref{claim:TBound} that $\min_{x'\in(x+\eta_p\BB)\cap K_{d_p}}\dsT^\infty_p(x')\leq\tilde{\dsT}^\infty_p(x)+\eta_p\leq\dsT^*(x)+\eta_p$.
        This completes the proof.
    \end{proof}
    
    The following claim shows the $\dsT^*(x)$ is upper bounded by the limit of $\dsT^\infty_p(x)$.
    \begin{claim}\label{claim:upperbound}
        For any $x\in\DD(\dsT^*)$, it holds that 
        \begin{align*}
            \dsT^*(x)\leq\liminf_{p\to+\infty}\min_{x'\in(x+\eta_p\BB)\cap K_{d_p}}\dsT^\infty_p(x').
        \end{align*}
    \end{claim}
    \begin{proof}
        Following the argument in the proof of Corollary~3.7 in \cite{cardaliaguet1999set}, we have $\Limsup_{p\to+\infty}\Epi(\tilde{\dsT}^\infty_p)\subseteq\Epi(\dsT^*)$.
        By the definition of $\Limsup$, we have $\forall (x, t)\in\Limsup_{p\to+\infty}\Epi(\tilde{\dsT}^\infty_p)$, $\liminf_{p\to+\infty}\dist((x, t), \Epi(\tilde{\dsT}^\infty_p))=0$.
        In other words, there exists a sequence $\{(x_p, t_p)\in\Epi(\tilde{\dsT}^\infty_p)\}$ s.t. $\liminf_{p\to+\infty}\dist((x, t), (x_p, t_p))=0$.
        
        Fix $x\in\DD(\dsT^*)$ and consider a sequence $\{(x_p, t_p)\}$ s.t. $x_p=x, t_p=\tilde{\dsT}^\infty_p(x_p)$.
        It follows from Claim~\ref{claim:lowerbound} that $t_p\leq\dsT^*(x)<+\infty$, where the last inequality follows from the definition of $\DD(\dsT^*)$.
        Therefore, $\{(x_p, t_p)\}$ is bounded.
        Then it follows from the supplementary statement about the definition of $\Limsup$ on page 236 in \cite{cardaliaguet1999set} that $\mathrm{Limsup}_{p\to+\infty}\{(x_p, t_p)\}\neq\emptyset$.
        Since $(x_p, t_p)\in\Epi(\tilde{\dsT}^\infty_p)$, we have 
        \begin{align*}
            \Limsup_{p\to+\infty}\{(x_p, t_p)\}\subseteq\Limsup_{p\to+\infty}\Epi(\tilde{\dsT}^\infty_p)\subseteq\Epi(\dsT^*).
        \end{align*}
        Let $(x^*, t^*)\in\Limsup_{p\to+\infty}\{(x_p, t_p)\}$ s.t. $t^*\leq t, \forall (x, t)\in\Limsup_{p\to+\infty}\{(x_p, t_p)\}$.
        Therefore, $t^*\geq\dsT^*(x^*)=\dsT^*(x)$.
        Then we have $\liminf_{p\to+\infty}\tilde{\dsT}^\infty_p(x)=t^*\geq\dsT^*(x)$.
        It follows from \eqref{eq:TLowerbound} in Claim~\ref{claim:TBound} that 
        \begin{align*}
            \dsT^*(x)\leq\liminf_{p\to+\infty}\min_{x'\in(x+\eta_p\BB)\cap K_{d_p}}\dsT^\infty_p(x').
        \end{align*}
        This completes the proof.
    \end{proof}
    Combining Claims~\ref{claim:lowerbound} and \ref{claim:upperbound} renders at 
    \begin{align*}
        &\min_{x'\in(x+\eta_p\BB)\cap K_{d_p}}\dsT^\infty_p(x')-\eta_p\leq\dsT^*(x)\\\leq&\liminf_{p\to+\infty}\min_{x'\in(x+\eta_p\BB)\cap K_{d_p}}\dsT^\infty_p(x').
    \end{align*}
    By taking $p\to+\infty$ on the left side, both $\eta_p$ and $d_p$ disappear at the same rate and we complete the proof.
\end{proof}

We proceed to show the convergence of the value function estimate on each point $x\in K_{d_p}$.
Different from Theorem~\ref{theorem:pointwiseConvergence}, we remove the perturbation $x'\in(x+\eta_p\BB)\cap K_{d_p}$ in the theorem statement and directly analyze the values $\dsT_p^\infty(x), \forall x\in K_{d_p}$.
Towards this end, we first prove the following lemma quantifying the lowerbound of $\dsT^*(x), \forall x\in K_{d_p}$.
\begin{lemma}\label{lemma:estimatesUpperBound}
    Given all conditions in Theorem~\ref{theorem:pointwiseConvergence}, the following relationship holds:
    \begin{align*}
        \dsT^*(x)\geq \dsT^\infty_p(x)-d_p, \forall x\in K_{d_p}.
    \end{align*}
\end{lemma}
\begin{proof}
    It follows from Proposition~2.21 in \cite{cardaliaguet1999set} that $(Viab_F(\HH)+d_p\BB)\cap\HH_{d_p}$ is a discrete viability domain for $\Gamma_{\epsilon_p, d_p}$.
    Thus, it holds that $(Viab_F(\HH)+d_p\BB)\cap\HH_{d_p}\subseteq Viab_{\Gamma_{\epsilon_p, d_p}}(\HH_{d_p})$.
    It follows from Theorem~3.2 and Lemma~3.6 that we can replace the discrete viability kernals with epigraphs of traveling time functions; i.e., 
    \begin{align*}
        (\Epi(\dsT^*)+d_p\BB)\cap\HH_{d_p}\subseteq\Epi(\dsT^\infty_p).
    \end{align*}
    Fix $x\in K_{d_p}$ and there exists $t\in\dsT^*(x)+d_p\BB$ s.t. $(x, t)\in\Epi(\dsT^\infty_p)$; that is, $t\geq \dsT^\infty_p(x)$.
    This implies $\dsT^*(x)+d_p\geq \dsT^\infty_p(x)$.
    This completes the proof.
\end{proof}
The following corollary characterized the pointwise convergence of the value function estimate.
\begin{corollary}\label{corollary:pointwiseConvergenceOnGrid}
    Given the conditions in Theorem~\ref{theorem:pointwiseConvergence} and $x\in K_{d_{\bar{p}}}\cap\DD(\dsT^*)$ for some $\bar{p}\geq1$. 
    Then the following holds for every $x\in K_{d_{\bar{p}}}\cap\DD(\dsT^*)$:
    \begin{align*}
        \dsT^*(x)=\lim_{p\to+\infty}\dsT^\infty_p(x).
    \end{align*}    
\end{corollary}
\begin{proof}
    The proof generally follows the arguments towards Theorem~\ref{theorem:pointwiseConvergence}.
    For $\tilde{T}_p^\infty$ defined as $\Epi(\tilde{T}_p^\infty)=\Epi(\dsT_p^\infty)+\eta_p\BB$, $\tilde{T}_p^\infty(x)\leq \dsT_p^\infty(x)$.
    Then it follows from Claim~\ref{claim:upperbound} that $\liminf_{p\to+\infty}\dsT_p^\infty(x)\geq\liminf_{p\to+\infty}\tilde{T}_p^\infty(x)\geq\dsT^*(x)$.
    It follows from Lemma~\ref{lemma:estimatesUpperBound} that $\dsT^\infty_p(x)\leq\dsT^*(x)+d_p$.
    Taking limits $p\to+\infty$, $d_p\to0$ completes the proof.
\end{proof}

Notice that Theorem~\ref{theorem:pointwiseConvergence} only characterizes the pointwise convergence over $\DD(\dsT^*)$.
On $K\setminus\DD(\dsT^*)$, the value of $\dsT^*$ is not well defined and the convergence cannot be analyzed.
The Kruzhkov transform can address this issue and let $\varTheta^*_k(x)=1, \forall x\in K\setminus\DD(\dsT^*)$.
Without proof, the theorem is stated below.
\begin{lemma}\label{lemma:transformedPointwiseConvergence}
    Given the convergence in the epigraphic sense, i.e., $\Epi(\dsT^*)=\Lim_{k\to+\infty}\Epi(\dsT^*_k)$, $\varTheta^*_k$ converges pointwise to $\varTheta^*$, i.e., $\forall x\in\XX$, 
    \begin{align*}
        \varTheta^*(x)=\lim_{k\to+\infty}\min_{x'\in(x+\ell_k\BB)\cap V_k}\varTheta^*_k(x'),
    \end{align*}
    where $\lim_{k\to+\infty}\frac{\ell_k}{d_k}\in[1, +\infty)$ and $\ell_k\geq d_k$.
\end{lemma}
The following corollary shows the uniform convergence for the transformed value function estimates.
\begin{corollary}\label{corollary:uniformConvergence}
    The transformed value function estimate $\{\varTheta^*_k\}$ converges uniformly to $\varTheta^*$ almost everywhere on $\XX$; i.e., 
    \begin{align*}
        \lim_{k\to+\infty}\|\varTheta^*-\min_{x'\in(x+\ell_k\BB)\cap V_k}\varTheta^*_k(x')\|_{\XX\setminus\ZZ}=0,
    \end{align*}
    where $\ZZ$ has sufficiently small measure, $\lim_{k\to+\infty}\frac{\ell_k}{d_k}\in[1, +\infty)$ and $\ell_k\geq d_k$.
\end{corollary}
\begin{proof}
    The proof leverages Egoroff's theorem \cite{royden2010real}.
    Since $\XX$ is compact, $\XX$ has finite measure.
    It follows from Theorem~\ref{theorem:pointwiseConvergence} that $\{\varTheta^*_k\}$ are measurable functions converging on $\XX$ pointwise to $\varTheta^*$.
    Then $\{\varTheta^*_k\}$ converges uniformly to $\varTheta^*$ almost everywhere to $\XX\setminus\ZZ$ by the Egoroff's theorem, where $\ZZ$ has sufficiently small measure.
    This completes the proof of the first relation.
\end{proof}
\end{document}